\documentclass{article}
\usepackage{ijcai21}
\usepackage{graphicx}
\usepackage{tikz}
\usetikzlibrary{bayesnet}

\usepackage{pifont}

\usepackage{times}
\usepackage{soul}
\usepackage{url}
\usepackage[hidelinks]{hyperref}
\usepackage[utf8]{inputenc}
\usepackage[small]{caption}
\usepackage{graphicx}
\usepackage{amsmath, amsfonts}
\usepackage{mathtools}
\usepackage{amsthm}
\usepackage{bbm}
\usepackage{booktabs}
\usepackage[ruled, vlined, linesnumbered, noend]{algorithm2e}
\DontPrintSemicolon
\SetKwInput{KwInput}{Input}
\SetArgSty{textnormal}
\usepackage{algorithmic}
\usepackage{thm-restate}
\usepackage{color}
\usepackage{natbib}
\usepackage{listings}
\usepackage{makecell}
\usepackage{subcaption}
\usepackage{multirow}
\allowdisplaybreaks

\usepackage{bm}

\newcommand{\softbold}{\bm}
\newcommand{\hardbold}{\pmb}

\newcommand{\vars}{\softbold{V}}
\newcommand{\varsvalue}{\softbold{v}}
\newcommand{\var}{V}
\newcommand{\varvalue}{v}
\newcommand{\parents}{\softbold{U}}
\newcommand{\parentsvalue}{\softbold{u}}
\newcommand{\parentvalue}{u}
\newcommand{\intvVar}{W}
\newcommand{\intvVars}{\bm{W}}
\newcommand{\intvVarvalue}{w}

\newcommand{\features}{\softbold{X}}
\newcommand{\featuresvalue}{\softbold{x}}
\newcommand{\feature}{X}
\newcommand{\featurevalue}{x}
\newcommand{\hiddenvars}{\softbold{H}}

\newcommand{\target}{Y}

\newcommand{\prediction}{\hat{Y}}
\newcommand{\predictionvalue}{\hat{y}}

\newcommand{\edges}{\softbold{E}}
\newcommand{\bn}{\mathcal{N}}
\newcommand{\graph}{\mathcal{G}}
\newcommand{\allParam}{\softbold{\Theta}}
\newcommand{\cpt}{\softbold{\theta}}
\newcommand{\param}{\theta}

\newcommand{\ind}{\lambda}
\newcommand{\allInds}{\softbold{\lambda}}

\newcommand{\ac}{\mathcal{AC}}
\newcommand{\allParamAC}{\softbold{\Phi}}
\newcommand{\paramAC}{\phi}

\newcommand{\intvSet}{\mathcal{I}}

\newcommand{\subcircuit}{\alpha}

\newcommand{\evidence}{\bm{e}}

\newcommand{\ub}[3]{UB(#1, #2, #3)}
\newcommand{\ubn}{UB }
\newcommand{\lbn}{LB }

\newcommand{\boolcircuit}{\Sigma}
\newcommand{\cnfencoding}{\Delta}
\newcommand{\interVars}{\bm{T}}
\newcommand{\interVar}{T}

\newcommand{\prob}{p}

\newcommand{\df}{F}

\newcommand{\context}{C}

\newcommand{\Pa}{\textup{pa}}

\newcommand{\introb}{\texttt{IntRob}}

\newcommand{\struc}{structural }
\newcommand{\Struc}{Structural }

\newcommand{\augbn}{\df}

\newtheorem{definition}{Definition}

\newtheorem{lemma}{Lemma}
\newtheorem{proposition}{Proposition}

\title{Provable Guarantees on the Robustness of Decision Rules to Causal Interventions}

\author{
Benjie Wang\footnote{Equal contribution}\and
Clare Lyle\footnotemark[\value{footnote}]\And
Marta Kwiatkowska\\
\affiliations
University of Oxford\\
\emails
benjie.wang@cs.ox.ac.uk
}

\begin{document}

\maketitle

\begin{abstract}
Robustness of decision rules to shifts in the data-generating process is crucial to the successful deployment of decision-making systems. Such shifts can be viewed as interventions on a causal graph, which capture (possibly hypothetical) changes in the data-generating process, whether due to natural reasons or by the action of an adversary. We consider causal Bayesian networks and formally define the \textit{interventional robustness} problem, a novel \textit{model-based} notion of robustness for decision functions that measures worst-case performance with respect to a set of interventions that denote changes to parameters and/or causal influences. By relying on a tractable representation of Bayesian networks as arithmetic circuits, we provide efficient algorithms for computing guaranteed upper and lower bounds on the interventional robustness probabilities. Experimental results demonstrate that the methods yield useful and interpretable bounds for a range of practical networks, paving the way towards provably causally robust decision-making systems. 
\end{abstract}

\section{Introduction}
\label{sec:intro}
As algorithmic decision-making systems become widely deployed, there has been an increasing focus on their safety and robustness, particularly when they are applied to input points outside of the data distribution they were trained on. 
Much of the work in this area has focused on instance-based robustness properties of classifiers, which guarantee that the prediction does not change in some vicinity of a specific input point \citep{Shih18Formal, NarodytskaEtAl18}. However, there are many types of distribution shift that cannot be characterized by robustness against norm-bounded perturbations to individual inputs. 
Such distribution shifts are often instead characterized by causal interventions on the data-generating process \citep{Quionero-CandelaEtAl09, ZhangEtAl15, LiptonEtAl18}. These interventions give rise to a range of different environments (distributions), which can be the effect of natural shifts (e.g. different country) or actions of other agents (e.g. a hospital changing prescription policy).

To assess the impact of such interventions, we must leverage knowledge about the causal structure of the data-generating distribution. This paper concerns itself with a simple question: given a decision-making system and a posited causal model, is the system robust to a set of plausible interventions to the causal model? 
Defining and verifying such \textit{model-based }notions of robustness requires a formal representation of the decision-making system. For discrete input features and a discrete output class, regardless of how a classifier is learned, its role in decision-making can be unambiguously represented by its \textit{decision function}, mapping features to an output class. 
This observation has spurred a recent trend of applying logic for meta-reasoning about classifier properties, such as monotonicity and instance-based robustness, by \textit{compiling} the classifier into a tractable form \citep{Shih18Formal, NarodytskaEtAl18, AudemardEtAl20}, for example an ordered decision diagram. 
We extend this approach to causal modelling by combining logical representations of the decision rule and causal model, and compiling this joint representation into an arithmetic circuit, a tractable representation of probability distributions.

Our main technical contributions are as follows. First, we motivate and formalize the robustness of a decision rule with respect to interventions on a causal model, which we call the \textit{interventional robustness problem}, and characterize its complexity.
Second, we develop a \textit{joint compilation} technique which allows us to reason about a causal model and decision function simultaneously. Finally, we develop and evaluate algorithms for computing upper and lower bounds on the interventional robustness problem, enabling the verification of robustness of decision-making systems to causal interventions.

\subsubsection{Related Work}

The problem of constructing classifiers which are robust to distribution shifts has received much attention from the machine learning perspective \citep{Quionero-CandelaEtAl09, ZhangEtAl15, LiptonEtAl18}. 
Particularly relevant to our work are \textit{proactive} approaches to learning robust classifiers, which aim to produce classifiers that perform well across a range of environments (rather than a specific one) \citep{Rojas-CarullaEtAl18, SubbaswamyEtAl19}. 

A recent line of work analyses the behaviour of machine learning classifiers using symbolic and logical approaches by compiling these classifiers into suitable logical representations \citep{KatzEtAl17, NarodytskaEtAl18, ShiEtAl20}. Such representations can be used to answer a range of explanation and verification queries \citep{AudemardEtAl20} about the classifier tractably, depending on the properties of the underlying propositional language \citep{Darwiche02Map}. Our work uses this premise to tackle defining and verifying robustness to distribution shift, which involves not only the classifier but also a probablistic causal model such as a causal Bayesian network.

In the Bayesian network literature, \textit{sensitivity analysis} \citep{ChanDarwiche04} is concerned with examining the effect of (typically small) local changes in parameter values on a target probability. We are concerned with providing worst-case guarantees against a set of possible causal interventions, which can involve changing parameters in multiple CPTs, and even altering the graphical structure of the network. This requires new methods that enable scalability to these large, potentially structural intervention sets.
Our causal perspective generalizes and extends the work of \citet{qin2015differential}, considering a richer class of interventions than previous work and using this perspective to prove robustness properties of a decision function.

\section{Background and Notation}

In the rest of this paper, we use $\vars = (\features, \target, \hiddenvars)$ to denote the set of modelled variables, which includes observable features $\features$, the prediction target $\target$, and hidden variables $\hiddenvars$. We use lower case (e.g. $\featuresvalue$) to denote instantiations of variables.

\subsubsection{Decision Functions}

Consider the task of predicting $\target$ given features $\features$. Though many machine learning (ML) techniques exist for this task, once learned, the input-output behaviour of any classifier can be characterized by means of a symbolic decision function $\df$ from $\features$ to $\target$. For many important classes of ML methods, including Bayesian network classifiers, binarized neural networks, and random forests, it is possible to encode the corresponding decision function as a Boolean circuit $\boolcircuit$ \citep{AudemardEtAl20, NarodytskaEtAl18, ShihEtAl19}. Such logical encodings can then be used to reason about the behaviour of the decision function, for instance providing explanations for decisions and verifying properties.

\subsubsection{Causal Bayesian Networks}

In this paper, we are interested in robust performance of decision functions under distribution shift caused by changes in the data-generating process (DGP). In order to reason about this, we first need a \textit{causal model} of the DGP which enables such changes to be represented. We first define Bayesian networks, which are a convenient way to specify a joint distribution over the set of variables $\vars = \{\var_1, \var_2, ..., \var_n\}$:

\begin{definition} [Bayesian Network]
A (discrete) \underline{Bayesian network (BN)} $\bn$ over variables $\bm{V}$ is a pair $(\graph, \allParam)$, where $\graph = (\vars, \edges)$ is a directed acyclic graph (DAG) whose nodes correspond to the random variables $\vars$ and whose edges indicate conditional dependence, and where $\allParam$ denotes the set of conditional probability tables (CPTs) $\cpt_{\var_i|\parents_i}$ with parameters $\param_{\varvalue_i|\parentsvalue_i} = P(\var_i = \varvalue_i |\parents_i = \parentsvalue_i)$ which specify the distribution, where $\parents_i = \Pa_{\graph}(\var_i)$ are the parents of $\var_i$ in $\graph$. We will denote by $\prob_{\bn}$ the probability distribution defined by the BN $\bn$.
\end{definition}

Causal Bayesian networks (CBNs) are defined similarly to Bayesian networks, with the addition of \textit{causal}, or \textit{interventional}, semantics to the joint distribution. Intuitively, an edge $(\var, \var')$ in a causal Bayesian network indicates that $\var$ \textit{causes} $\var'$, and the CPTs correspond to causal mechanisms. An intervention can be defined to be a change to some of these mechanisms, replacing $\allParam$ with $\allParam'$.
A CBN can thus be characterized as representing a set of distributions, each of which is generated by a different intervention.

We now define a representation of a (causal) Bayesian network, called the \textit{network polynomial}, based on the seminal work of \citet{Darwiche00}. 
\begin{definition} [Network Polynomial]
The \underline{network polynomial} of causal BN $\bn$ is defined to be:
\begin{equation}
    l_{\bn}[\allInds, \allParam] = \sum_{\varvalue_1, ..., \varvalue_n} \prod_{i = 1}^{n} \ind_{\varvalue_i} \param_{\varvalue_i|\parentsvalue_i}
\end{equation}
where $\ind_{\varvalue_i}$ denotes an \emph{indicator variable} for each value $\varvalue_i$ in the support of each random variable $\var_i$, and $\param_{\varvalue_i|\parentsvalue_i}$ denotes each element of a CPT in $\allParam$. Each component of the addition $l_{\varsvalue}[\allInds, \allParam] := \prod_{i = 1}^{n} \ind_{\varvalue_i} \param_{\varvalue_i|\parentsvalue_i}$ is called a \underline{term}, and is associated with an instantiation $\varsvalue$.
\end{definition}

\subsubsection{Arithmetic circuits}

Arithmetic circuits (AC) are computational graphs used to encode probability distributions over a set of discrete variables $\bm{V}$, which can tractably answer a broad range of probabilistic queries, depending on certain structural properties (called decomposability, smoothness and determinism). They were first introduced by \citet{Darwiche00} as a means of \textit{compiling} Bayesian networks for the purposes of efficient inference. Subsequently they have been considered as objects of study in their own right, with proposals for directly learning ACs from data \citep{Lowd08Learningac} and extensions relaxing determinism \citep{Poon11Spn}. 

\begin{definition}[Arithmetic Circuit]
    An \underline{arithmetic circuit} $\ac$ over variables $\bm{V}$ and with parameters $\allParamAC$ is a rooted directed acyclic graph (DAG), whose internal nodes are labelled with with $+$ or $\times$ and whose leaf nodes are labelled with indicator variables $\ind_\varvalue$, where $\varvalue$ is the value of some variable $\var \in \vars$, or non-negative parameters $\paramAC$. 
\end{definition}

Crucially, evaluating an arithmetic circuit can be done in time linear in the size (number of edges) of the circuit. When an AC represents a probability distribution, this means that marginals can be computed efficiently.

Like Bayesian networks, arithmetic circuits can be represented as polynomials over indicator and parameter variables, based on subcircuits \citep{ChoiDarwiche17}:

\begin{restatable}[Complete Subcircuit]{definition}{defSubcircuit} 
A \underline{complete subcircuit} $\subcircuit$ of an $\ac$ is obtained by traversing the circuit top-down, choosing one child of every visited $+$-node and all children of every visited $\times$-node. The \underline{term} $term(\subcircuit)$ of $\subcircuit$ is the product of all leaf nodes visited (i.e. all indicator and parameter variables). The set of all complete subcircuits is denoted $\hardbold{\subcircuit}_{\ac}$.
\end{restatable}

\begin{definition} [AC Polynomial]
The \underline{AC polynomial} of arithmetic circuit $\mathcal{AC}$ is defined to be:
$$l_{\ac}[\allInds, \allParamAC] = \sum_{\subcircuit \in \hardbold{\subcircuit}_{\ac}} term(\subcircuit)$$
\end{definition}

\section{The Intervention Robustness Problem}

\begin{figure}

    \includegraphics[width=0.96\linewidth]{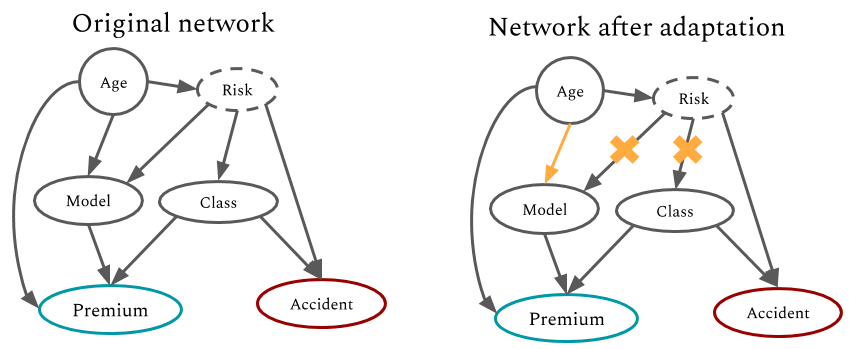}
\caption{An example causal model describing accident risk for a car insurance problem, and illustrating how strategic adaptation to a classifier can be characterized as a change to a causal model describing how the data was generated.}
\label{fig:insurance-intervention}
\end{figure}

Many distribution shifts faced by decision-making systems can be characterized by an \textit{intervention} on the data-generating process.
For example, if an insurance company gives reduced premiums to drivers who have taken a driving class, the relationship between `risk aversion' and `class' in Figure~\ref{fig:insurance-intervention} may change in response as more risk-seeking drivers take driving classes to benefit from reduced premiums. The company therefore seeks to determine whether this policy will be \textit{robust} to changes in the relationship between risk-sensitivity and driving classes before it deploys the policy.

We thus formulate an \textit{intervention robustness} problem which considers the worst-case drop in performance of a classifier in response to changes to a \emph{subset} of the causal mechanisms of the Bayesian network. This is inspired by the principle of \textit{independent causal mechanisms} (ICM) \citep{Peters17Book}, which states that causal mechanisms do not inform or influence each other; that is, even as some mechanisms are changed, other mechanisms tend to remain invariant. In the insurance example, this is reflected in that we would not necessarily expect the mechanism for 'risk aversion' or 'accident' to change, for instance.

While many related notions of robustness exist in the literature, none accurately captures this notion of robustness to causal mechanism changes. Many popular definitions of robustness measure the size of a perturbation necessary to change an input's classification, without taking into account that such perturbations may change the value which the classifier tries to predict \citep{Shih18Formal}. \citet{MillerEtAl20} highlight the connection between causal inference and robustness to distribution shifts caused by `gaming' in the \textit{strategic classification} \citep{HardtEtAl16} regime. However, \citet{MillerEtAl20} does not assume access to a known causal model, and its focus is on identifying classifiers which are robust to gaming, whereas our objective is to verify robustness to a much richer collection of distribution shifts.

\subsection{Intervention classes}

To reason about the effects of changes to a causal model, we need a formal description of these interventions. We consider interventions as \textit{actions} that modify the mechanisms of a causal Bayesian network $\bn = (\graph, \allParam)$, thereby changing its joint distribution. In particular, we consider two types of interventions: the first concerns changes to the parameters of the causal model, while the second concerns changes to the existence of cause-effect relationships themselves.

Typically, we might expect that only mechanisms for a subset of variables $\intvVars \subseteq \vars$ will change. In what follows, given a subset of variables $\intvVars \subseteq \vars$, we will use $\cpt_{\intvVars}^{(\graph)} \subseteq \allParam$ to denote the parameters associated with the CPTs for variables $\intvVar \in \intvVars$, where the parents of $\intvVar$ are given by graph $\graph$. 

\begin{definition}[Parametric Interventions]
A \underline{parametric intervention} on variables $\intvVars$ substitutes a subset of parameters $\cpt_{\intvVars}^{(\graph)}$ for new values $\cpt_{\intvVars}^{(\graph) \prime}$ obtaining a new parameter set $\allParam'$, which yields the BN:
\begin{equation}
    \bn [\cpt_{\intvVars}^{(\graph) \prime}] \coloneqq (\graph, \allParam')
\end{equation}
\end{definition}
Parametric interventions encompass the do-interventions discussed by \citet{qin2015differential}, but allow us to express more complex changes to causal mechanisms than fixing a variable to a set value. We can further consider changes not just to the parameters of the network, but also to its edge structure; such changes to a set of variables $\intvVars$ can be described by a \textit{context function} $\context_{\intvVars}: \intvVars \to \mathcal{P}(\vars)$, which replaces the parents of $\intvVar \in \intvVars$ in $\graph$ with $\context_{\intvVars}(\intvVar)$, producing a new graph $\graph'$.
We refer to such interventions as \textit{\struc} interventions.
In this work we restrict ourselves to context sets which preserve the acyclicity of the DAG.

\begin{definition}[\Struc interventions]
A \underline{\struc intervention} on variables $\intvVars$ modifies the edges $\edges$ of the graph $\graph = (\vars, \edges)$ according to a context function $\context_{\intvVars}$, obtaining a new graph $\graph' = (\vars, \edges')$, and substitutes parameters $\cpt_{\intvVars}^{(\graph)}$ for new values $\cpt_{\intvVars}^{(\graph') \prime}$, obtaining a new parameter set $\allParam'$, which yields the BN:
\begin{equation}
    \bn[\cpt_{\intvVars}^{(\graph') \prime}, \context_{\intvVars}] \coloneqq (\graph', \allParam')
\end{equation}
\end{definition}

We will often be interested in considering all of the possible interventions of a given class on some subset $\intvVars \subseteq \vars$ of the variables in the causal graph. 
Letting $\mathbb{P}_{\graph}(\intvVars)$ denote the set of valid parameter sets for $\intvVars \subseteq \vars$ in graph $\graph$, 
we will write for parametric interventions:
\begin{equation}
    \intvSet_{\bn}[\intvVars] := \{\bn[\cpt_{\intvVars}^{(\graph) \prime}] \; | \; \cpt_{\intvVars}^{(\graph) \prime} \in \mathbb{P}_{\graph}(\intvVars)\}
\end{equation}
and for \struc interventions,
\begin{equation}
    \intvSet_{\bn}[\intvVars, \context_{\intvVars}] := \{\bn[\cpt_{\intvVars}^{(\graph') \prime}, \context_{\intvVars}] \; | \; \cpt_{\intvVars}^{(\graph') \prime}\in \mathbb{P}_{\graph'}(\intvVars)\} .
\end{equation}

\subsection{Problem Definition and Complexity}
\label{sec:introb-definition}
Accurately assessing the robustness of a decision function $\df$ requires an understanding of the causal data-generating process. We propose to model this DGP using a causal Bayesian network $\bn$ on all variables $\vars$, thus enabling causal \textit{model-based} analysis of classifiers.
In order to reason about the causal structure $\bn$ and decision rule $\df$ simultaneously, we add an additional node to the CBN $\bn$.

\begin{definition} [Augmented BN]
For a CBN $\bn$ over variables $\vars$ and a classifier $\df: \features \to \target$, we define the \underline{augmented BN} $\bn_{\augbn}$ based on $\bn$ as follows: $\vars_{\augbn} = \vars \cup \{\prediction\}$ with $\Pa(\prediction) = \features$ and deterministic CPT  $ \theta_{\predictionvalue|\featuresvalue} = \mathbbm{1}[\predictionvalue = F(\featuresvalue)]$.
\end{definition}

This produces a well-defined joint distribution over the variables $\vars$ and $\prediction = \df(\features)$, which allows us to specify performance metrics as probabilities of events $\evidence$. For instance, a classifier's false positive rate can be expressed as the probability $p_{\bn_{\augbn}}(\evidence)$ of event $\evidence = (\prediction = 1) \wedge (\target = 0)$. More importantly, we can consider how these metrics change as the joint distribution changes, due to hypothetical or observed interventions on the causal model. This provides a basis for \textit{model-based} notions of robustness of decision rules.

We use intervention sets to represent all interventions which the modeller considers plausible. The \textit{interventional robustness} problem then concerns the worst-case performance of the decision rule over interventions in that set. 

\begin{definition}
Given CBN $\bn$ and decision rule $\df$, let $\intvSet_{\bn_{\augbn}}$ be an intervention set for the augmented BN $\bn_{\augbn}$, $\evidence$ be an assignment of a subset of the variables in $\vars$, and $\epsilon > 0$. The \underline{interventional robustness problem} is that of computing:
$$\introb(\intvSet_{\bn_{\augbn}}, \evidence) := \max_{\bn' \in \intvSet_{\bn_{\augbn}}} \prob_{\bn'}(\evidence)\;.$$
We also have the corresponding decision problem:
$$\introb(\intvSet_{\bn_{\augbn}}, \evidence, \epsilon) := \max_{\bn' \in \intvSet_{\bn_{\augbn}}} \prob_{\bn'}(\evidence) > \epsilon\; .$$
\end{definition}

We will be particularly interested in problem instances where $\intvSet_{\bn}$ is of the form $\intvSet_{\bn}[\intvVars]$, in which case we can view the problem instance as $\introb((\bn, \intvVars), \evidence)$. We observe that, via a reduction which we defer to the Appendix, the causal semantics of $\introb$ do not increase the computational hardness of the problem beyond that of MAP inference.

\begin{restatable}{theorem}{thmMAP}\label{thm:hardness} Let $\bn=(\graph, \allParam)$ be a causal Bayesian network, with $n$ nodes and maximal in-degree $d$.  Then an instance of MAP can be reduced to an instance of $\introb$ on a BN $\bn'$ of size linear in $|\bn|$, and of treewidth $w' \leq w + 2$. An instance of $\introb$ can be reduced to an instance of MAP on a BN $\bn'$ whose CPT $\allParam'$ has size polynomial in the size of $\allParam$, and with treewidth $w' \leq 2w$.
\end{restatable}

\section{Verification of Intervention Robustness}

In this section, we present our approach to verifying interventional robustness. 
Due to the difficulty of the problem, we seek to approximate $\introb(\intvSet, \evidence)$ by providing guaranteed upper and lower bounds that can be efficiently computed\footnote{In Appendix \ref{apx:worked_example}, we show with an example that heuristic approximation using existing methods cannot provide guarantees.}.

\subsection{Joint Compilation}

Our first goal is to \textit{compile} $\bn_{\augbn}$ into an equivalent arithmetic circuit $\ac$.
To do so, we make use of a standard CNF encoding $\cnfencoding_{\bn}$ of the causal BN $\bn$, defined over the indicator and parameter variables $\allInds_{\vars}, \allParam$ \citep{Chavira05Compiling}, and additionally an encoding of the decision function $\df$. 

A na\"ive encoding of $\df$ is to explicitly enumerate all instantiations of features $\featuresvalue$ and prediction $\predictionvalue$, and encode these directly as CNF clauses. However, this approach is very inefficient for larger feature sets $\features$. We instead assume access to an encoding of the classifier as a Boolean circuit $\Sigma$ over input features $\features$ and prediction $\prediction$ \footnote{If input features are discrete, they can still be encoded using additional binary variables; see Appendix \ref{apx:compile}.}. Such a circuit can be converted to CNF through the Tseitin transformation, introducing additional intermediate variables $\interVars$, obtaining a CNF formula $\cnfencoding_{\df}$ over $\allInds_{\features}, \allInds_{\prediction}, \interVars$.
We then combine the encodings of $\df$ and $\bn$ simply by conjoining the CNF formulae, to produce a new formula $\cnfencoding_{joint} = \cnfencoding_{\bn} \wedge \cnfencoding_{\df}$, over $\allInds_{\vars}, \allInds_{\prediction}, \allParam, \interVars$.

To construct an AC $\ac$, we now \textit{compile} this CNF encoding into d-DNNF (deterministic decomposable negation normal form), using the \textsc{C2D} compiler \citep{Darwiche04}, and then replace $\vee$-nodes with $+$, $\wedge$-nodes with $\times$, and set all negative literals and literals corresponding to $\interVars$ to $1$. This produces an AC with polynomial $l_\ac[\allInds, \allParam]$, where $\allInds := \allInds_{\vars} \cup \allInds_{\prediction}$. Crucially, this AC is equivalent to the augmented BN, in the following sense:

\begin{restatable}{proposition}{thmJoint}\label{thm:joint}
$l_\ac[\allInds, \allParam]$ is equivalent to $l_{\bn_{\augbn}}[\allInds, \allParam]$. Further, $\ac$ can be used to faithfully evaluate marginal probabilities $\prob_{\bn'}(\evidence)$ under any parametric intervention $\bn'$.
\end{restatable}

The time and space complexity of this procedure is $O(nw2^{w})$, where $n$ is the number of CNF variables and $w$ the treewidth, a measure of the connectivity of the CNF. When jointly compiling a BN and a decision function, we can bound $n, w$ in terms of the individual encodings $\cnfencoding_{\bn}, \cnfencoding_{\df}$.

\begin{restatable}{proposition}{thmSize} \label{thm:ac_size2}
    Suppose $\cnfencoding_{\bn}$ has $n$ variables and treewidth $w$, and $\cnfencoding_{\df}$ has $n'$ variables and treewidth $w'$. Then $\cnfencoding_{joint}$ has exactly $n + n' - |\allInds_{\features}|$ variables, and treewidth  at most $\max(w, w', \min(w, w') + |\allInds_{\features}|)$.
\end{restatable}

\subsection{Orderings} \label{sec:acorderings}

For the correctness of our upper bounding algorithm, it is necessary to impose some structural constraints on the circuit.

Firstly, any circuit compiled using the procedure described above has the property that every $+$-node $t$ has two children, and is associated with some CNF variable $c$, such that one child branch has $c$ true, and the other has $c$ false (information on the identity of this variable for each $+$-node is provided by the C2D compiler). We need to ensure that the arithmetic circuit only contains $+$-nodes associated with indicators $\allInds$, and not intermediate variables $\interVars$. Provided this is the case, the branches of each $+$-node $t$ will contain contradicting indicators for some unique variable $\var$. We can thus say that $t$ `splits' variable $\var$, as each of its child branches corresponds to different values of $\var$, and we write $split(t)$ to denote this splitting variable.

Secondly, provided the above holds, we require our circuit to satisfy some constraints of the following form.

\begin{restatable}{definition}{defOrdering}
    An arithmetic circuit $\ac$ \underline{satisfies the ordering constraint $(V_j, V_i)$} if:
    \begin{multline}
    \label{eq:constraint}
    \forall t, t', (split(t) = \var_i \wedge split(t') = \var_j) \\
    \implies \text{t' is not a descendant of t in $\ac$}
    \end{multline}
\end{restatable}

Intuitively, our algorithm requires that for BN variables in the intervention set $\intvVars$, the relative position of splitting $+$-nodes in the AC agrees with the causal ordering in the BN. More formally, we say that $\ac$ \ul{satisfies the ordering constraints associated with intervention set $\intvSet_{\bn_{\augbn}}$}, if for all $\var_i \in \intvVars$, and all $\var_j$ such that $\var_j \in \Pa_{\graph}(\var_i)$ (parametric intervention set) or $\var_j \in \Pa_{\graph'}(\var_i)$ (\struc intervention set), $\ac$ satisfies the ordering constraint $(V_j, V_i)$. In practice, when computationally feasible, we compile ACs with \textit{topological} and \textit{\struc topological} orderings, which satisfy these constraints for all $\var_i$, not just $\var_i \in \intvVars$; such orderings have the advantage of being valid for any intervention sets $\intvVars$.

We enforce these constraints by enforcing corresponding constraints on the \textit{elimination ordering} $\pi$ over the CNF variables, which is used to construct the \textit{dtree} that is used in the compilation process, and affects the time and space taken by the compilation. Such an elimination ordering is usually chosen using a heuristic such as min-fill. We instead find a elimination ordering by using a \textit{constrained} min-fill heuristic, which ensures that these constraints are satisfied, but may produce an AC which is much larger than can be achieved with an unconstrained heuristic in practice. We provide further details of this in Appendix \ref{apx:elim}.

\subsection{Upper bounds on Intervention Robustness}
\begin{figure}[t]
\begin{subfigure}{\linewidth}
\centering
\tikz{
    \node[obs] (X) {$X$};
    \node[obs, xshift = 2cm] (W) {$W$};
    \node[obs, below = of X, xshift = 1cm] (Yh) {$\hat{Y}$};
    \node[obs, below = of W, xshift = 1cm] (Y) {$Y$};
    
    \edge {X} {Yh};
    \edge {W} {Yh};
    \edge {W} {Y};
    \draw[dashed, ->] (X) -- (W);
}
\label{fig:sub1}
\end{subfigure}\\[1ex]
\begin{subfigure}{.48\linewidth}
\centering
\begin{tikzpicture}[scale=0.8,transform shape,wrap/.style={inner sep=0pt,
fit=#1,transform shape=false}] %
        \node[draw=none] (P1) {+};
        
        \node[draw=none, inner sep = 0.05cm, below=0.1cm of P1, xshift=-1cm] (M1) {*};
        \node[draw=none, inner sep = 0.05cm, below=0.1cm of P1, xshift=1cm] (M2) {*};
        
        \node[draw=none, below=0.15cm of M1, inner sep = 0.05cm, xshift=0.2cm, text=red] (P2) {+};
        \node[draw=none, below=0.15cm of M2, inner sep = 0.05cm, xshift=-0.2cm, text=red] (P3) {+};
        \node[draw=none, below=0.2cm of M1, inner sep=0cm, xshift=-0.3cm] (Tx) {$\theta_x$};
        \node[draw=none, below=0.2cm of M1, inner sep=0cm, xshift=-0.9cm] (Lx) {$\lambda_x$};
        \node[draw=none, below=0.2cm of M2, inner sep=0cm, xshift=0.3cm] (Txn) {$\theta_{\bar{x}}$};
        \node[draw=none, below=0.2cm of M2, inner sep=0cm, xshift=0.9cm] (Lxn) {$\lambda_{\bar{x}}$};
        
        \node[draw=none, below=0.3cm of P2,inner sep=0.05cm, xshift=-0.3cm] (M3) {*};
        \node[draw=none, below=0.3cm of P2, inner sep=0.05cm, xshift=0.3cm] (M4) {*};
        \node[draw=none, below=0.3cm of P3, inner sep=0.05cm, xshift=-0.3cm] (M5) {*};
        \node[draw=none, below=0.3cm of P3, inner sep=0.05cm, xshift=0.3cm] (M6) {*};
        
        \node[draw=none, below=0.1cm of M3, inner sep=0cm, xshift=-0.9cm] (Lyh) {$\lambda_{\hat{y}}$};
        \node[draw=none, below=0.1cm of M6, inner sep=0cm, xshift=0.9cm] (Lyhn) {$\lambda_{\bar{\hat{y}}}$};
        \node[draw=none, below=1.2cm of M3, inner sep=0cm, xshift=-0.3cm] (Lw) {$\lambda_{w}$};
        \node[draw=none, below=1.2cm of M6, inner sep=0cm, xshift=0.3cm] (Lwn) {$\lambda_{\bar{w}}$};
        
        \node[draw=none, below=1cm of M3, inner sep=0cm, xshift=-0.9cm, text=red] (Tw) {$\theta_{w}$};
        \node[draw=none, below=1cm of M6, inner sep=0cm, xshift=0.9cm, text=red] (Twn) {$\theta_{\bar{w}}$};
        
        \node[draw=none, below=1.5cm of M3, inner sep=0.05cm, xshift=0.3cm] (P4) {+};
        \node[draw=none, below=1.5cm of M6, inner sep=0.05cm, xshift=-0.3cm] (P5) {+};
        
        \node[draw=none, below=0.3cm of P4, inner sep=0.05cm, xshift=-0.3cm] (M7) {*};
        \node[draw=none, below=0.3cm of P4, inner sep=0.05cm, xshift=0.3cm] (M8) {*};
        \node[draw=none, below=0.3cm of P5, inner sep=0.05cm, xshift=-0.3cm] (M9) {*};
        \node[draw=none, below=0.3cm of P5, inner sep=0.05cm, xshift=0.3cm] (M10) {*};
        
        \node[draw=none, below=0.8cm of M7, inner sep=0cm, xshift=0.4cm] (Ly) {$\lambda_{y}$};
        \node[draw=none, below=0.8cm of M10, inner sep=0cm, xshift=-0.4cm] (Lyn) {$\lambda_{\bar{y}}$};
        
        \node[draw=none, below=0.1cm of M7, inner sep=0cm, xshift=-0.3cm] (Tyw) {$\theta_{y|w}$};
        \node[draw=none, below=0.1cm of M8, inner sep=0cm, xshift=0cm] (Tynw) {$\theta_{\bar{y}|w}$};
        \node[draw=none, below=0.1cm of M9, inner sep=0cm, xshift=-0cm] (Tywn) {$\theta_{y|\bar{w}}$};
        \node[draw=none, below=0.1cm of M10, inner sep=0cm, xshift=0.3cm] (Tynwn) {$\theta_{\bar{y}|\bar{w}}$};
        
        \draw[-] (P1) -- (M1);
        \draw[-] (P1) -- (M2);
        
        \draw[-] (M1) -- (P2);
        \draw[-] (M1) -- (Tx);
        \draw[-] (M1) -- (Lx);
        \draw[-] (M2) -- (P3);
        \draw[-] (M2) -- (Txn);
        \draw[-] (M2) -- (Lxn);
        
        \draw[-] (P2) -- (M3);
        \draw[-] (P2) -- (M4);
        \draw[-] (P3) -- (M5);
        \draw[-] (P3) -- (M6);
        
        \draw[-] (M3) -- (P4);
        \draw[-] (M4) -- (P5);
        \draw[-] (M5) -- (P4);
        \draw[-] (M6) -- (P5);
        
        \draw[-] (M3) -- (Tw);
        \draw[-] (M4) -- (Twn);
        \draw[-] (M5) -- (Tw);
        \draw[-] (M6) -- (Twn);
        
        \draw[-] (M3) -- (Lyh);
        \draw[-] (M4) -- (Lyh);
        \draw[-] (M5) -- (Lyh);
        \draw[-] (M6) -- (Lyhn);
        \draw[-] (M3) -- (Lw);
        \draw[-] (M4) -- (Lwn);
        \draw[-] (M5) -- (Lw);
        \draw[-] (M6) -- (Lwn);
        
        \draw[-] (P4) -- (M7);
        \draw[-] (P4) -- (M8);
        \draw[-] (P5) -- (M9);
        \draw[-] (P5) -- (M10);
        
        \draw[-] (M7) -- (Ly);
        \draw[-] (M9) -- (Ly);
        \draw[-] (M8) -- (Lyn);
        \draw[-] (M10) -- (Lyn);
        
        \draw[-] (M7) -- (Tyw);
        \draw[-] (M8) -- (Tynw);
        \draw[-] (M9) -- (Tywn);
        \draw[-] (M10) -- (Tynwn);
        
   \end{tikzpicture}

\label{fig:sub2}
\end{subfigure}
\begin{subfigure}{.48\linewidth}
\centering
\begin{tikzpicture}[scale=0.8,transform shape,wrap/.style={inner sep=0pt,
fit=#1,transform shape=false}] %
\tikzstyle{every node}=[font=\small]
        \node[draw=none] (P1) {$0.4$};
        
        \node[draw=none, inner sep = 0.05cm, below=0.1cm of P1, xshift=-1cm] (M1) {$0.3$};
        \node[draw=none, inner sep = 0.05cm, below=0.1cm of P1, xshift=1cm] (M2) {$0.1$};
        
        \node[draw=none, below=0.15cm of M1, inner sep = 0.05cm, xshift=0.2cm, text=red] (P2) {$0.6$};
        \node[draw=none, below=0.15cm of M2, inner sep = 0.05cm, xshift=-0.2cm, text=red] (P3) {$0.2$};
        \node[draw=none, below=0.2cm of M1, inner sep=0cm, xshift=-0.3cm] (Tx) {$0.5$};
        \node[draw=none, below=0.2cm of M1, inner sep=0cm, xshift=-0.9cm] (Lx) {$1$};
        \node[draw=none, below=0.2cm of M2, inner sep=0cm, xshift=0.3cm] (Txn) {$0.5$};
        \node[draw=none, below=0.2cm of M2, inner sep=0cm, xshift=0.9cm] (Lxn) {$1$};
        
        \node[draw=none, below=0.3cm of P2,inner sep=0.05cm, xshift=-0.3cm] (M3) {$0.2$};
        \node[draw=none, below=0.3cm of P2, inner sep=0.05cm, xshift=0.3cm] (M4) {$0.6$};
        \node[draw=none, below=0.3cm of P3, inner sep=0.05cm, xshift=-0.3cm] (M5) {$0.2$};
        \node[draw=none, below=0.3cm of P3, inner sep=0.05cm, xshift=0.3cm] (M6) {$0$};
        
        \node[draw=none, below=0.1cm of M3, inner sep=0cm, xshift=-0.9cm] (Lyh) {$1$};
        \node[draw=none, below=0.1cm of M6, inner sep=0cm, xshift=0.9cm] (Lyhn) {$0$};
        \node[draw=none, below=1.2cm of M3, inner sep=0cm, xshift=-0.3cm] (Lw) {$1$};
        \node[draw=none, below=1.2cm of M6, inner sep=0cm, xshift=0.3cm] (Lwn) {$1$};
        
        \node[draw=none, below=1cm of M3, inner sep=0cm, xshift=-0.9cm, text=red] (Tw) {$1$};
        \node[draw=none, below=1cm of M6, inner sep=0cm, xshift=0.9cm, text=red] (Twn) {$1$};
        
        \node[draw=none, below=1.5cm of M3, inner sep=0.05cm, xshift=0.3cm] (P4) {$0.2$};
        \node[draw=none, below=1.5cm of M6, inner sep=0.05cm, xshift=-0.3cm] (P5) {$0.6$};
        
        \node[draw=none, below=0.3cm of P4, inner sep=0.05cm, xshift=-0.3cm] (M7) {$0$};
        \node[draw=none, below=0.3cm of P4, inner sep=0.05cm, xshift=0.3cm] (M8) {$0.2$};
        \node[draw=none, below=0.3cm of P5, inner sep=0.05cm, xshift=-0.3cm] (M9) {$0$};
        \node[draw=none, below=0.3cm of P5, inner sep=0.05cm, xshift=0.3cm] (M10) {$0.6$};
        
        \node[draw=none, below=0.8cm of M7, inner sep=0cm, xshift=0.4cm] (Ly) {$0$};
        \node[draw=none, below=0.8cm of M10, inner sep=0cm, xshift=-0.4cm] (Lyn) {$1$};
        
        \node[draw=none, below=0.1cm of M7, inner sep=0cm, xshift=-0.3cm] (Tyw) {$0.8$};
        \node[draw=none, below=0.1cm of M8, inner sep=0cm, xshift=0cm] (Tynw) {$0.2$};
        \node[draw=none, below=0.1cm of M9, inner sep=0cm, xshift=-0cm] (Tywn) {$0.4$};
        \node[draw=none, below=0.1cm of M10, inner sep=0cm, xshift=0.3cm] (Tynwn) {$0.6$};
        
        \draw[-] (P1) -- (M1);
        \draw[-] (P1) -- (M2);
        
        \draw[-] (M1) -- (P2);
        \draw[-] (M1) -- (Tx);
        \draw[-] (M1) -- (Lx);
        \draw[-] (M2) -- (P3);
        \draw[-] (M2) -- (Txn);
        \draw[-] (M2) -- (Lxn);
        
        \draw[-] (P2) -- (M3);
        \draw[-] (P2) -- (M4);
        \draw[-] (P3) -- (M5);
        \draw[-] (P3) -- (M6);
        
        \draw[-] (M3) -- (P4);
        \draw[-] (M4) -- (P5);
        \draw[-] (M5) -- (P4);
        \draw[-] (M6) -- (P5);
        
        \draw[-] (M3) -- (Tw);
        \draw[-] (M4) -- (Twn);
        \draw[-] (M5) -- (Tw);
        \draw[-] (M6) -- (Twn);
        
        \draw[-] (M3) -- (Lyh);
        \draw[-] (M4) -- (Lyh);
        \draw[-] (M5) -- (Lyh);
        \draw[-] (M6) -- (Lyhn);
        \draw[-] (M3) -- (Lw);
        \draw[-] (M4) -- (Lwn);
        \draw[-] (M5) -- (Lw);
        \draw[-] (M6) -- (Lwn);
        
        \draw[-] (P4) -- (M7);
        \draw[-] (P4) -- (M8);
        \draw[-] (P5) -- (M9);
        \draw[-] (P5) -- (M10);
        
        \draw[-] (M7) -- (Ly);
        \draw[-] (M9) -- (Ly);
        \draw[-] (M8) -- (Lyn);
        \draw[-] (M10) -- (Lyn);
        
        \draw[-] (M7) -- (Tyw);
        \draw[-] (M8) -- (Tynw);
        \draw[-] (M9) -- (Tywn);
        \draw[-] (M10) -- (Tynwn);
        
   \end{tikzpicture}
\label{fig:sub3}
\end{subfigure}
\caption{Example augmented BN (top), corresponding AC (bottom left), and execution of Algorithm \ref{alg:bottom_up} (bottom right). Nodes which differ from standard AC evaluation are highlighted in red.}
\label{fig:ub_example}
\end{figure}
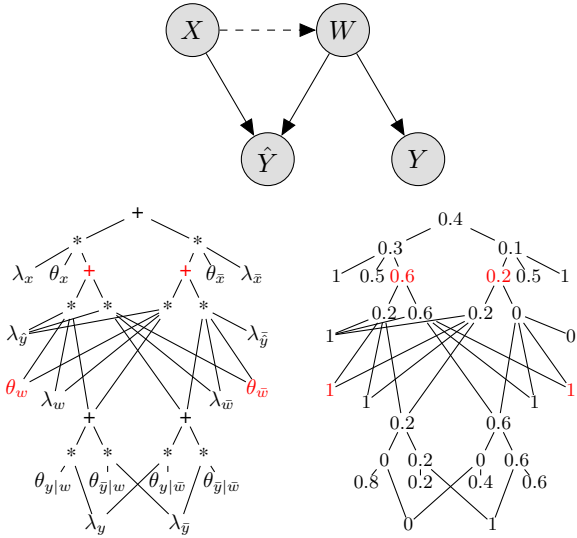
In order to compute upper bounds on the interventional robustness quantity, we propose Algorithm \ref{alg:bottom_up}, which sets parameters in the AC for the CPTs of variables in $\intvVars$ to 1, and applies maximization instead of addition at $+$-nodes splitting on variables in $\intvVars$, when evaluating the (appropriately ordered) AC. Algorithm \ref{alg:bottom_up} somewhat resembles the well-known MPE algorithm on ACs, introduced by \citet{ChanDarwiche06} and used as an upper bound on the MAP problem in \citet{HuangEtAl06}. However, our algorithm maximizes over parameters rather than variables and makes use of specific AC structure ensured by our ordering constraints; the reason it produces correct upper bounds is thus also different.

\begin{algorithm}[t]
\SetAlgoLined
\KwInput{$\ac$, the AC; evidence $\evidence$; intervenable variables $\intvVars \subseteq \vars$;
}
\KwResult{Output probability $p$}

\For{\text{node} $c \in \ac$ (children before parents)} {
	\Switch{type($c$)}{
        \Case{Indicator $\ind_v$}{
            $p[c] := 0$ \textbf{if} $\varvalue'$ not consistent with $\evidence$ \textbf{else} $1$
        }
        \Case{Parameter $\theta_{\varvalue|\parentsvalue}$}{
            \label{algp:set_param} $p[c] := 1$ \textbf{if} $\var \in \intvVars$\textbf{else} $\theta_{\varvalue|\parentsvalue}$
        }
        \Case{$\times$} {
        	$p[c] := \prod_{d} p[d]$ \text{ where $d$ are the children of $c$}
        }
        \Case{$+$} {
        	\If{$c$\text{ splits on some $\intvVar \in \intvVars$}} {
                $p[c] := \max_{d} p[d]$ \text{ where $d$ are the children of $c$} 	
            }
            \Else{
				$p[c] := \sum_{d} p[d]$ \text{ where $d$ are the children of $c$}          
            }
        }
    }
}

\textbf{Return} $p[c_{root}]$, where $c_{root}$ is the root node of $\ac$ 

\caption{$\ub{\ac}{\evidence}{\intvVars}$ (Upper Bounding)}
\label{alg:bottom_up}
\end{algorithm}

Intuitively, the maximizations represent decision points, where choosing a child branch corresponds to intervening to set a particular parameter $\param_{\intvVarvalue|\parentsvalue_{\intvVars}}$ to 1 (and others to 0). For example, consider the augmented Bayesian network in Figure \ref{fig:ub_example}, where all variables are binary, $\prediction = \feature \vee \intvVar$, and the context for $\intvVar$ is $\feature$ (represented by the dashed line, which is not part of the original BN). Figure \ref{fig:ub_example} also shows execution of Algorithm \ref{alg:bottom_up} for false positive probability (that is, evidence $\evidence = \{\target = 0, \prediction = 1\}$). At the two $+$-nodes where maximizations occur, the value of $\feature$ is already ``decided'', and the adversary can effectively choose to set $\param_{\bar{\intvVarvalue}|\featurevalue} = 1$ and $\param_{\intvVarvalue|\bar{\featurevalue}} = 1$. In this case, the result $0.4$ turns out to be exactly equal to the interventional robustness quantity $\introb(\intvSet_{\bn_{\augbn}}, \evidence)$.

We might ask whether this intuition is correct in general. Our next result shows that, while the algorithm cannot always compute $\introb(\intvSet_{\bn_{\augbn}}, \evidence)$ exactly, it does produce guaranteed upper bounds (proof in Appendix \ref{apx:ub_proof}):

\begin{restatable}{theorem}{thmUB} \label{thm:ub_correctness}
Given a parametric/\struc intervention set $\intvSet_{\bn_{\augbn}}$, let $\ac$ be an arithmetic circuit with the same polynomial as $\bn_{\augbn}$, and satisfying the ordering constraints associated with the intervention set. Then, applying the \ubn  algorithm $\ub{\ac}{\evidence}{\intvVars}$ returns a quantity $B_U$ which is an upper bound on the interventional robustness quantity $\introb(\intvSet_{\bn_{\augbn}}, \evidence)$.
\end{restatable}

This result is quite surprising; it shows that it is possible, through a very simple and inexpensive procedure requiring just a single linear time pass through the AC, to upper bound the worst-case marginal probability over an exponentially sized set of interventions. 
That this also holds for \struc intervention sets, which alter the structure of the Bayesian network which the AC was compiled from, is even more surprising.
Further, a compiled AC can be used for \emph{any} intervention set given that it satisfies the appropriate ordering constraints. For instance, an AC compiled using a topological ordering allows us to derive upper bounds for parametric intervention sets involving any subset (of any size) $\intvVars \subseteq \vars$, simply by setting the appropriate parameter nodes in the AC to 1 (Line \ref{algp:set_param}). This allows us to amortize the cost of evaluating robustness against multiple intervention sets.

\subsection{Lower bounds via best-response dynamics}

In addition to an upper bound on $\introb(\intvSet_{\bn_{\df}}, \evidence)$, we can also straightforwardly lower bound this quantity using any witness. In this section, we will assume the setting of parametric interventions of the form $\intvSet_{\bn}[\intvVars]$, or of \struc interventions where the search is over a single context function.

\begin{algorithm}

\SetAlgoLined
\KwInput{$\bn=(\graph, \allParam)$, a Bayesian network; evidence $\evidence$ whose probability will be maximized; intervenable variables $\intvVars \subseteq \vars$. 
}
\KwResult{Output probability $p(e) \leq \max_{\bn' \in {\intvSet}_{\bn}[\intvVars]} P_{\bn'}(\evidence)$}
\Begin{
$v \gets 0$;\\
\While{$p_{\bn[\allParam_{\intvVars}]}(\evidence) > v$}{
	    $v \gets p_{\bn[\allParam_{\intvVars}]}(\evidence)$; \\
\For{CPT $\cpt^{(\graph)}_{W|u} \in \allParam_{\intvVars}$ } {
    $\cpt^{(\graph)}_{W|\parentsvalue} \gets \text{arg max}_{\cpt'_{W|\parentsvalue}} p_{
\mathcal{N}[ \cpt'_{W|\parentsvalue}]}(\evidence)$;\\
   }
}
}
\caption{Lower Bounding}
\label{alg:brdynamics}
\end{algorithm}

We obtain such an approach by formalizing the problem of finding an intervention which maximizes $\prob_{\bn'}(\evidence)$ as a multiplayer game, where each instantiation  $\parentsvalue_{\intvVar}$ of parents $\Pa_{\graph'}(\intvVar)$ for each $\intvVar \in \intvVars$ specifies a player, and where all players share a utility function given by $\prob_{\bn[\allParam]}(\evidence)$. Each player's strategy set consists of the set of deterministic conditional distributions $\cpt_{W|\parentsvalue}$ (we note w.l.o.g. that, by the multilinearity of the network polynomial, the optimal value of $P_{\bn[\allParam']}(e)$ is obtained by at least one deterministic interventional distribution). A Nash equilibrium in this game then corresponds to an interventional distribution for which no change in a single parameter can increase $\prob_{\bn'}(\evidence)$. Algorithm~\ref{alg:brdynamics} follows best-response dynamics in this game. 
We provide an analysis of the time complexity and convergence of this approach in the proof of the following proposition.

\begin{restatable}{proposition}{thmLB}\label{thm:lb-correctness}
Algorithm \ref{alg:brdynamics} converges to a locally optimal parametric intervention in finite time. Further, if the algorithm is stopped before termination, the current value $v$ will be a lower bound on $\max_{\bn' \in \intvSet[\intvVars]}P_{\bn'}(\evidence)$.
\end{restatable}

\section{Case Study: Insurance}
In this case study, we look at an extended version of the car insurance example, using the Bayesian network model shown in Figure \ref{fig:insurancebn} \citep{Binder97Insurance}.

\begin{figure}[htb]
    \centering
    \includegraphics[scale=0.45]{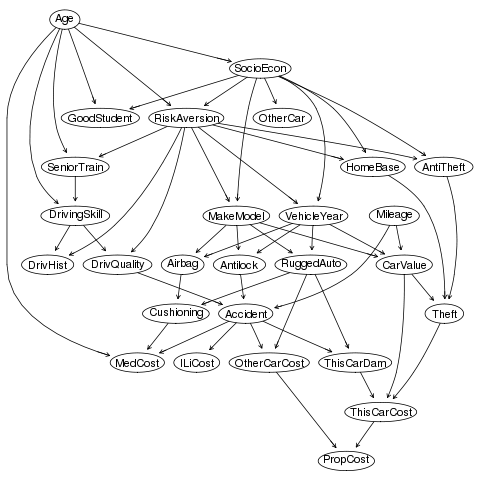}
    \caption{INSURANCE Bayesian network}
    \label{fig:insurancebn}
\end{figure}

Suppose an insurance company wishes to use a classifier to predict \texttt{MedCost} (the medical cost of an insurance claim), and has access to an insurant's \texttt{Age}, \texttt{DrivHist}, and \texttt{MakeModel} (categorical variables with 3-5 values). \texttt{MedCost} is either \textit{BelowThousand} (0) or \textit{AboveThousand} (1). They fit a Na\"ive Bayes classifier to historical data, obtaining a decision function $F$. This is then used as part of their decision-making policy determining what premiums to offer to customers. 

The company is particularly concerned about false negatives, as this could result in the company losing a lot of money in payouts. Based on the original Bayesian network model (Figure \ref{fig:insurancebn}) and their new classifier, this should occur $2.5\%$ of the time. However, in reality, insurants may attempt to game the classifier to predict \textit{BelowThousand} (so that they get lower premiums), while actually being likely to have a high medical cost. In our framework, we model this using structural interventions, assuming that insurants can causally intervene on some of \texttt{DrivHist} (perhaps hide some accident history), \texttt{MakeModel} (choose a different type of car than they would normally choose), and \texttt{Cushioning} (upgrade/downgrade the degree of protection inside the car). The company would thus like to understand how \emph{robust} their classifier is to these adaptations.

We will consider a number of structural intervention sets $\intvSet_{\bn_{\df}}$, given by intervenable variables $\intvVars$, which may be any subset of $\{\texttt{DrivHist, MakeModel, Cushioning}\}$. Under each of these intervention sets, we seek to obtain guaranteed upper bounds on these two quantities:
\begin{itemize}
    \item \textbf{FN}: The probability of a false negative $p(F = 0, \texttt{MedCost} = 1)$, i.e. predicted low medical cost, but high actual medical cost.
    \item \textbf{P}: The probability of a positive $p(\texttt{MedCost} = 1)$, i.e. high actual medical cost.
\end{itemize}

\begin{table}[]
\centering
\begin{tabular}{lll}
\toprule
\textbf{Intervenable Variables $\intvVars$} & \textbf{FN} & \textbf{P} \\ \midrule
Empty Set                                     & 2.5\%       & 7.2\%      \\
\{DrivHist\}                                  & 7.2\%       & 7.2\%      \\ 
\{MakeModel\}                                 & 5.7\%       & 10.0\%     \\ 
\{Cushioning\}                                & 6.1\%       & 12.9\%     \\ 
\{DrivHist, MakeModel\}                       & 10.0\%      & 10.0\%     \\ 
\{DrivHist, Cushioning\}                      & 12.9\%      & 12.9\%     \\ 
\{MakeModel, Cushioning\}                     & 13.0\%      & 13.9\%     \\ 
\{DrivHist, MakeModel, Cushioning\}           & 13.9\%      & 13.9\%     \\ \bottomrule
\end{tabular}

\caption{Guaranteed upper bounds on FN and P, under different structural intervention sets}
\label{tab:insurance_fnp}
\end{table}

The results are shown in Table \ref{tab:insurance_fnp}. The insurance company can use these bounds to assess risk, and improve their classifier's robustness if they deem the false negative rate under intervention unacceptable. 

The bounds can also provide further insight. We notice that whenever \texttt{DrivHist} is intervenable, the percentage of false negatives is the same as positives, i.e. the classifier always predicts wrong when \texttt{MedCost} is $1$. This turns out to be because the Na\"ive Bayes classifier always predicts $0$ whenever \texttt{DrivHist} is \textit{None}, regardless of the other input variables. Thus, an insurant who can change their \texttt{DrivHist} can always fool the classifier to predict $0$. In addition, the percentage of positives doesn't increase from the original BN: this can be seen from the causal graph, where \texttt{DrivHist} has no causal influence on \texttt{MedCost}. 

On the other hand, \texttt{Cushioning} significantly increases the positive rate. Notice that, in the graph, intervening on \texttt{Cushioning} will not have any influence on the inputs to the classifier; thus, the increase in FN to $6.1\%$ is not due to fooling the classifier, but rather making high medical expenses generally more likely, by downgrading the quality of cushioning. In this way, the intervention is "taking advantage" of the classifier not having full information about cushioning.

\section{Evaluations}

\subsubsection{Compilation Performance}
In Table \ref{tab:ac_size_preview} we show the performance of our joint compilation approach on a number of benchmark Bayesian networks, where we jointly compile the network and a decision rule (see Appendices \ref{apx:additional-evals} and \ref{apx:exp_details} for further details). We observe that the sizes of the compiled ACs are significantly smaller than the worst-case bounds would suggest (exponential in treewidth). Further, when we enforce a topological ordering, the size of the compilation increases, but not by more than $\sim 100$. Our results provide evidence that our methods can scale to fairly large networks and classifiers, including networks compiled with topological and \struc topological orderings.

\begin{table}[htb] 
\begin{tabular}{lllrrr}
\toprule
\textbf{Net} & \textbf{CSize} & \textbf{Ord} & \textbf{TW} & \textbf{AC size} & \makecell{\textbf{Time} \\ \textbf{(s)} }  \\ \midrule
insurance         &  3 (41)                            & N       &    24      & 167121          & 0.5                         \\ 
         & 3 (41)                            & T       &    31     & 794267         & 4  \\
         & 3 (41)                            & S       &    33     & 1270075         & 8  \\ \midrule
win95pts         &  16 (799)                            & N        &    51    & 1210072          & 3                         \\ 
         & 16 (799)                            & T        &    58     & 52266950          & 77                         \\ \midrule
hepar2         & 12 (946)                            & N        &    53    & 8096874          & 49                         \\ 
        & 12 (946)                            & T        &    51     & 123108407          & 73                         \\
        & 12 (946)                            & S        &    51     & 123164181          & 75                         \\
\bottomrule
\end{tabular}
\caption{AC sizes and times (s) for the joint compilations used in the \ubn and \lbn algorithms. Shown are the number of input features $d$ and the sizes of the Boolean circuits representing the classifier, ordering constraints (none, topological, or \struc topological), treewidth of the combined CNF encoding, and AC size and compilation time. Further evaluations provided in Appendix~\ref{apx:additional-size}.}
\label{tab:ac_size_preview}
\end{table}

\begin{table}[htb]
\centering
\begin{tabular}{lllll}
\toprule
\textbf{Network} & \textbf{IntSet} &  \textbf{LBound} & \textbf{UBound} & $\bm{\Delta}$\\ \midrule
insurance        & P1                           &       0.1181            & 0.1276      & \textbf{0.0095}   \\ 
       & P2                            &       0.3275            & 0.3433    &     \textbf{0.0158}  \\ 
        & S1                          &       0.1181            & 0.1297       &  \textbf{0.0116} \\ \midrule
win95pts         & P1                            &       0.2111            & 0.2111    &    \textbf{0.0000}    \\ 
        & P2                          &       0.2163            & 0.2191       &   \textbf{0.0028}  \\ \midrule
hepar2         & P1                & 0.09445     & 0.09445 & \textbf{0.0000}        \\ 
        & P2                 & 0.09585     & 0.09585 & \textbf{0.0000}          \\
        &  S1              &  0.1029    & 0.1029 & \textbf{0.0000}       \\

\bottomrule
\end{tabular}
\caption{Analysis of the tightness of bounds (on probability of false negatives) produced by Algorithms \ref{alg:bottom_up} and \ref{alg:brdynamics}. For each network, we have different intervention sets (P/S indicates the intervention set is parametric/structural respectively). 
Lower and upper bounds, along with the difference, are shown for each intervention set. Further evaluations provided in Appendix~\ref{apx:tightness}.
}

\label{tab:tightness}
\end{table}

\subsubsection{Lower and Upper Bound tightness}
\label{sec:tightness}

In Table~\ref{tab:tightness} we analyse the quality of our upper and lower bounds on interventional robustness. We compute bounds on false negative probability under different intervention sets. Overall, we find small or nonexistent gaps between the lower and upper bounds across all networks and intervention sets evaluated, suggesting that in many settings of interest it is possible to obtain tight guarantees using our algorithms. 

Further, both bounding algorithms are very fast to execute, taking no more than a few seconds for each run. This is remarkable given the sizes of the intervention sets. For instance, for the insurance network, the parametric intervention set P2 covers 6 variables ($|\intvVars| = 6$), 248 parameters, and $\sim 10^{36}$ different interventions, making brute-force search clearly infeasible. 
For worst-case (interventional robustness) analysis, the sensitivity analysis method of \citet{ChanDarwiche04} requires $\sim 10^7$ passes through the AC in this case. On the other hand, our upper bounding algorithm requires an ordered AC (which is $\sim 5$ times larger in this case), but requires just a \emph{single} pass through the AC, making it $\sim 10^6$ faster. Further, our algorithm is uniquely able to provide guarantees for structural intervention sets.

\section{Conclusions}
In this work, we have motivated and formalized the interventional robustness problem, developed a compilation technique to produce efficient joint representations for classifiers and DGPs, and provided tractable upper and lower bounding algorithms which we have shown empirically to be tight on a range of networks and intervention sets. The techniques presented here provide ample opportunity for further work, such as extending the upper and lower bounding technique to networks where the modeller has uncertainty over the parameters, and developing learning algorithms for arithmetic circuits which permit reasoning about causal structure. 

\paragraph{Acknowledgements}
This project was funded by the ERC under the European Union’s Horizon 2020 research and innovation programme (FUN2MODEL, grant agreement No. 834115).

\bibliographystyle{named}
\bibliography{causalRobust}

\appendix

\section{Proof of Theorem \ref{thm:hardness}}

We now prove Theorem \ref{thm:hardness}. To do so, we will formalize the intervention maximization problem, whose decision form is IntRob, as follows.

$\introb$ is the Intervention Robustness problem. Note that while the definition given in Section~\ref{sec:introb-definition} is agnostic to the form of the interventions, in order to concretely bound the size of the input to the problem, we restrict ourselves to Bayesian networks defined on discrete random variables and intervention sets of the form $\intvSet_{\bn}(\intvVars)$. Thus, the problem takes as input a Bayesian network $\bn$, a set of intervenable nodes $\intvVars$, and some evidence $\evidence$. The goal is to find a parametric intervention on the variables in $\intvVars$ $\allParam'$ such that $P_{\bn[\allParam']}(\evidence)$ is maximized. The decision version of this problem is intervention robustness.

\textbf{MAP}: is Maximum a Posteriori inference, a well-studied problem which takes as input network $\bn$, variables $\intvVars$, and evidence $\evidence$, and whose objective is to find an instantiation $\intvVars = w$ such that $P_{\bn}(w, \evidence)$ is maximal.

\thmMAP*
Without loss of generality, in the proof of this statement we will assume all variables are binary-valued; it is straightforward to then extend the results here to arbitrary discrete random variables supported on a finite set. 
\begin{lemma}\label{lem:map2introb}
MAP is reducible to \introb.
\end{lemma}
\begin{proof}
Let $\bn, \intvVars, \evidence$ be an instantiation of the MAP problem. We can convert this into the $\introb$ problem by adding the following sets of nodes to $\bn$ in order to produce a new network $\bn'$. 
\begin{enumerate}
\item For each $V \in \intvVars$, add a node $V_\theta$ with the same support as $V$, and which has no incoming arrows.
\item Additionally, for each $V \in \intvVars$ add a node $A_V$ with parents $V$ and $V_\theta$ with support True/False, which is True with probability 1 if $V = V_\theta$ and False otherwise. Let $\mathbf{A}_{\intvVars}$ denote the set of all such nodes.
\end{enumerate}
\begin{figure}
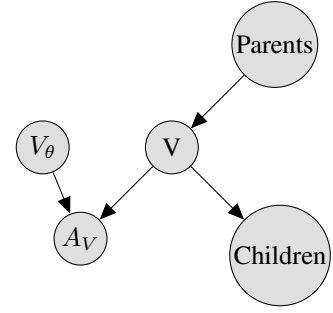

    \centering
      \tikz{ %
  	
        \node[obs] (V) {V} ; 
        \node[obs, left=of  V] (V1) {$V_\theta$} ; %
        \node[obs, below left=of V](A){$A_V$}; %
        \node[obs, above right=of V] (U) {Parents};
        \node[obs, below right=of V](C){Children};

		\edge {V} {A};
		\edge{V1}{A};
		\edge{U}{V};
		\edge{V}{C};
		
      }
    \caption{Visualization of the construction of $\bn'$ for the proof of Lemma~\ref{lem:map2introb}.}
    \label{fig:map2introb}
\end{figure}
\begin{figure}
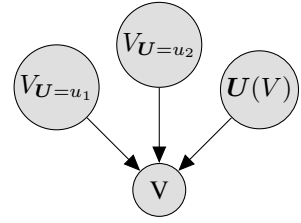

    \centering
      \tikz{ %
  	
        \node[obs] (V) {V} ; 
        \node[obs, above left=of  V] (V1) {$V_{\parents=u_1}$} ; %
        \node[obs, above =of  V] (V2) {$V_{\parents=u_2}$} ; %
        \node[obs, above right=of V] (U) {$\parents(\var)$};
        
		\edge{U}{V};
		\edge{V2}{V};
		\edge{V1}{V};
		
      }
    \caption{Visualization of the intuition behind the construction of $\bn'$ for the proof of Lemma~\ref{lem:introb2map}.}
    \label{fig:intuition}
\end{figure}

We now show that in this new network $\bn'$, \introb($\bn'$, $\intvVars_\theta$, $\{\mathbf{A_V}=$ True, $\evidence\}$) is equal to MAP($\bn$, $\intvVars$, $\evidence$). 

We first observe that for a single $V \in \intvVar$, we have 
\begin{align*}
    P_{\bn}(V=v) &= P_{\bn[\cpt_{V_\theta}=\mathbbm{1}\{V_\theta=v\}]}(V = V_\theta)) \\
    &=P_{\bn[\cpt_{V_\theta}=\mathbbm{1}\{V_\theta=v\}}(A_V=\text{True}))\; .
\end{align*}
Because $V_\theta$ and $A_V$ are independent of the rest of the graph given $V$, we then straightforwardly obtain that for additional evidence $\evidence$, the same equality holds for the joint evidence $(V=v, e)$:
\begin{equation*}
    P_{\bn[\cpt_{V_\theta}=\mathbbm{1}\{V_\theta=v\}}(A_V=\text{True}, \evidence) = P_{\bn}(V=v, \evidence) \; .
\end{equation*} 
Finally, this equality can be iterated to incorporate all nodes $V \in \intvSet$, and so for any instantiation $\mathbf{w}=(\mathbf{w}^1, \dots \mathbf{w}^n)$ of $\intvVars = \{V^1, \dots, V^n\}$ with corresponding parameters $\allParam'_{\mathbf{w}} = \{\mathbbm{1}[V^i_\theta = \mathbf{w}^i] | V^i \in \intvVars \}$  we obtain
\begin{equation*}
     P_{\bn[\allParam_{\mathbf{w}}']}(\mathbf{A}_{\intvVars}=\text{True}, \evidence) = P_{\bn}(\intvVars=\mathbf{w}, \evidence) \; .
\end{equation*}
So the parametric interventions which maximize $P(\mathbf{A_{\intvVars}}=T, \evidence)$ are equivalent to the variable values $\bm{w}$ which maximize $P(\intvVars=\bm{w}, \evidence)$. 

It is straightforward to show that the size of the resulting BN $\bn'$ satisfies the theorem statement. We have increased the number of nodes by $2|\intvVars| \leq 2 | \vars |$, and added CPTs of size $|\text{supp}(V)|^2$, which for binary variables will be fixed at 4, so $|\allParam '| \leq |\allParam | + 4 |\vars|$. Finally, we have not increased the treewidth of the network by more than a constant increment of 2 because we have added a fully connected component with 2 additional nodes to each variable $V \in \intvSet$ with no other edges into the graph.
\end{proof}

\begin{lemma}
\label{lem:introb2map}
$\introb$ is reducible to MAP.
\end{lemma}
\begin{proof}
\medskip

For the opposite direction, we show that we can use MAP to solve \introb (with parametric interventions). Let $\bn, \intvVars, \evidence$ be inputs into \introb. We initially construct $\bn'$ as a copy of $\bn$. We will proceed by converting the parameters $\cpt_{V|u}$ into variables $V_u$ in $\bn'$, where observing the value $V_u=v$ in $\bn'$ is equivalent, up to a constant factor, to setting $\cpt_{V|u} := \mathbbm{1}[V=v]$ in the original network $\bn$, while avoiding an exponential blowup in the size of this new Bayesian network.

\medskip
\textbf{Intuition:} For each CPT component $\cpt_{V|u}$ in $\bn$ such that $V \in \intvVars$, introduce a new ``parameter-node" $V_{u}$ to $\bn'$ with support equal to that of $V$ and uniformly distributed (i.p. $V_u$ has no parents), and add an edge into $V$. Set $P(V=v' |u, V_u) = \mathbbm{1}[ V_{u} = v']$. In other words, if the value of $U = u$, then $V$ gets the value of the node $V_u$ deterministically. This construction is visualized in Figure~\ref{fig:intuition}.

Naively, this will blow up the CPT for $V$ by a factor exponential in the size of the parent set, and may increase the treewidth of the network by this factor as well. We therefore proceed in the following construction to minimize the impact on the size of the representation of $\bn'$. This construction will still add in the worst case $|\allParam|$ variables to the network, but it will only increase the size of $| \allParam|$ by a linear factor and will only increase the treewidth by at most a factor of 2. 

\medskip 

\textbf{Selector circuits:} in a more efficient (and more involved) construction, we add a number of auxiliary variables that act as a filter based on the values of the parents $U$ to pass down the correct value to $V$. We visualize an example of this construction in Figure~\ref{fig:selector}. We enumerate the parents of $V$ as $U_1, \dots, U_n$, and add auxiliary variables $\{V_{\parentsvalue} | \parentsvalue \in \text{supp}(\parents)\}$ to the graph as described in the intuition above. We then add auxiliary variables $S_{ij}$, with $i$ ranging from 1 to $n$ and $j$ ranging from $1$ to $2^{|U| -i}$ for each level $i$. To each variable $V_u$ we assign it the binary string $u_n \dots u_1$. For each prefix $u \in \{0,1\}^{n-1}$ we then draw arrows from $U_1$, $V_{u0}$ and $V_{u1}$ to $S_{1, u}$and define the conditional $P( S_{1, u} | V_{u0}, V_{u1}, U_1=u_1) = \mathbbm{1}[S_{1,u} = V_{u, u_1}]$. We inductively define at layer $i$ for a prefix $u \in \{0,1\}^{i-1}$ the random variable $S_{i,u}$, with parent variables $S_{i-1, u0}, S_{i-1, u1}$, and $U_{i}$, and conditional distribution $P(S_{i, u} | S_{i-1, u0}, S_{i-1, u1}, U_i=u_i) = \mathbbm{1}[S_{i,u} = V_{u, u_i}]$. The value of $S_{n, \emptyset}$ will therefore be deterministically the value of $V_u$ for $U = u$, and so we can simply set the CPT of $V$ to depend uniquely and deterministically on $S_{n, \emptyset}$.

This procedure will not add more nodes than there are parameters in the CPT to the graph (i.e. the increase in the number of nodes $n$ is bounded by $\exp(w)$) and will not increase the treewidth by more than a linear factor, as the selector circuit has treewidth at most $2|\parents| + 1$ (this is easily observed by forming a tree decomposition via a depth-first-search procedure). The maximal number of nodes this construction can add to the table is therefore $2|\allParam|$, assuming that all $n$ nodes are set to be intervened on, and we can increase the treewidth by a factor of at most 2, independent of the number of nodes we modify. We therefore obtain that for $\bn'$: $|\vars| \leq |\vars| + 2 | \allParam |$, $| \allParam | \leq 16 | \allParam |$, and $w' \leq 2w$.

The end result of this construction is that we have changed the distribution of the variable $V$ so that it now depends on $2^{|\parents|}$ additional random variables $V_{\parentsvalue_1}, \dots, V_{\parentsvalue_k}$, and deterministically takes the value of $V_{u_i}$ whenever its parents satisfy $U = u_i$. 
\medskip 
\begin{figure}
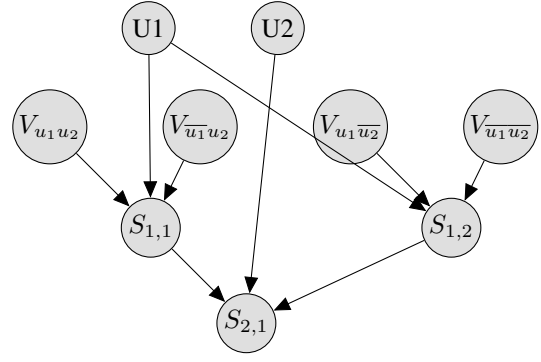

\centering
\tikz{

	\node[obs] (U1) {U1};
	\node[obs, right=of U1] (U2){U2};
	\node[obs, below left=of U1] (V11) {$V_{u_1u_2}$};
	\node[obs, right=of V11] (V12){$V_{\overline{u_1}u_2}$};
	\node[obs, right=of V12] (V21) {$V_{u_1\overline{u_2}}$};
	\node[obs, right=of V21] (V22){$V_{\overline{u_1}\overline{u_2}}$};
	\node[obs, below right=of V11](A1){$S_{1,1}$};
	\node[obs, below right=of V21](A2){$S_{1,2}$};
	\node[obs, below right=of A1](A3){$S_{2,1}$};

	\edge{U1}{A1};
	\edge{V11}{A1};
	\edge{V12}{A1};
	\edge{U1}{A2};
	\edge{V21}{A2};
	\edge{V22}{A2};
	\edge{A1}{A3};
	\edge{A2}{A3};
	\edge{U2}{A3};
	
}
\caption{Example `selector' circuit. At each level $i$, pair up solutions from paths that differ only on variables $u_1, \dots u_i$ (will be uniquely 2 for binary valued variables), then use the value of $u_i$ to `select' the correct value. I.e. $P(S_{1,1} = V_{u1, u2}|U_1) = 1$ if $U_1 = u_1$ else $P(A_1 = V_{\bar{u_1}u_2}) = 1$ if $U_1 = \bar{u_1}$.  After $|U|$ levels, the value of $A_n$ will be the value of $V_{u}$, and so will be used to substitute for $V$.}
\label{fig:selector}
\end{figure}

We now claim that in the new network $\bn'$, the conditional distribution on $V$ induced by observing $V_{\parentsvalue } = v_{\parentsvalue}$ is equal up to a constant to doing a parametric intervention on $\theta_{v|u}$ in $\bn$ which deterministically sets $v$ to $v_{\parentsvalue}$ conditioned on $\parents = \parentsvalue$. To see this, let $\mathbf{v_U}$ be an instantiation of the variables $\mathbf{V_U}$.
\begin{align*}
	P_{\bn'}(V=v|\parentsvalue,  \mathbf{v_U}) &= P_{\bn'} [V = v| V_u = v_u, \parents= \parentsvalue] \\
	&= \mathbbm{1}[v = v_u] \\
	&= P_{\bn[\theta_{v|u}=1]}[V = v|\parents = \parentsvalue] 
\end{align*}

We also observe that for any ancestor $U$ of $V$, $P_{\bn}(U) = P_{\bn'}(U)$, since we only changed the conditional distribution of a descendent of $\parents$. Further, for any descendent $D$ of $V$, we obtain $P_{B}(D | \parents = \parentsvalue, V = v) = P_{\bn'}(D|V=v, \parents=\parentsvalue)$. This is again because we did not change any conditional distribution for a descendent of $V$. 

Because $V_U$ is uniformly distributed in $\bn'$, we can therefore straightforwardly derive that for a single variable, $\introb$ can be reduced to MAP on $\bn'$. We decompose the evidence variables $\evidence$ into $\evidence_C$ and $\evidence_A$ (for descendents, ancestors of $V$ respectively). 
\begin{align}
&\max_{\cpt_{V|u}} P_{\bn[\cpt_{V|u}]}(\evidence) = \max_{\cpt_{V|u}} P_{\bn}(\evidence_A) P_{\bn[\cpt_{V|u}]}(\evidence_C ) \\
= & P_{\bn}(\evidence_A) \max_{v_u}  \sum_{u } P(\parentsvalue|\evidence_A) P(\evidence_C|V=v_u, \parentsvalue) \\
= &P_{\bn'}(\evidence_A) \max_{\mathbf{V_U}} 2^{|U|} P_{\bn'}(\mathbf{V_U} = \mathbf{v_U})  \\ &\sum_{u} P_{\bn'}(U = u|\evidence_A) P(\evidence_C|\evidence_A, u, V=v_{\parentsvalue}) \\
=&  2^{|\parents|} \max_{\mathbf{v_U}} P_{\bn'}(\evidence, \mathbf{V_U} = \mathbf{v_U} )
\end{align}

We note that because $2^{|\parents|}$ is a constant, it does not affect the maximization problem and so the two maximization problems will be maximized by equivalent parameter settings. Once the result is established for single variables, this observation is easily extended to the entire set $\intvVars$ by performing an analogous summation over all instantiations of variables which are parents of an element of $\intvVars$. The only trick is that, when summing over values $u$ of the parents of $\intvVars$, one must be careful in dealing with elements $v \in \intvVars$ which are also parents of intervenable variables. In such cases, we define $P_{\bn'}(\textbf{w}|\parentsvalue)$ to be zero if the elements at the intersection of $\intvVars$ and $\parents$ disagree on their assigned values.

\begin{align}
&\max_{\cpt_{\intvVars}} P_{\bn[\cpt_{\intvVars}]}(\evidence) = \max_{\cpt_{\intvVars}} P_{\bn}(\evidence_A) P_{\bn[\cpt_{\intvVars}]}(\evidence_C ) \\
= & P_{\bn}(\evidence_A) \max_{\mathbf{w}_{\parents}}  \sum_{\parentsvalue } P(\parentsvalue|\evidence_A) P(\evidence_C|\intvVars=\mathbf{w}_u, \parentsvalue) \\
= &P_{\bn'}(\evidence_A) \max_{\mathbf{w_U}} 2^{|\parents(\intvVars)|} P_{\bn'}(\mathbf{W_U} = \mathbf{w_U})  \\ &\sum_{u} P_{\bn'}(\parents = \parentsvalue|\evidence_A) P(\evidence_C|\evidence_A, \parentsvalue, V=v_{\parentsvalue}) \\
=&  2^{|\parents|} \max_{\mathbf{v_U}} P_{\bn'}(\evidence, \mathbf{V_U} = \mathbf{v_U} )
\end{align}
As noted previously, it is sufficient to consider parametric interventions that are deterministic to solve the $\introb$ problem (by the multilinearity of the network polynomial, the optimal value of $P_{\bn[\allParam']}(e)$ is obtained by at least one deterministic interventional distribution). Therefore, searching over the set of distributions induced by deterministic parametric interventions is equivalent to searching over conditional distributions induced by observing $\mathbf{V_U}$, and so this set of conditional distributions is sufficient for solving \introb.

\end{proof}

\section{Compilation Details and Proof of Proposition \ref{thm:joint}} \label{apx:compile}

Our goal with compilation is to produce an arithmetic circuit with AC polynomial equivalent to the network polynomial of $\bn_{\augbn}$. To do this, we first generate logical encodings of the decision function $\df$ and original Bayesian network $\bn$.

Any decision function $\df$ given as a Boolean circuit $\boolcircuit$ can be expressed as a CNF encoding $\cnfencoding_{\df}(\allInds_{\features}, \allInds_{\prediction}, \interVars)$ on indicator variables and intermediate variables. For discrete features $\features$ and discrete prediction $\prediction$, we assume that the Boolean circuit takes as input Boolean indicator variables $\ind_{\feature_i = \featuresvalue_i}, \ind_{\prediction = \predictionvalue}$, which outputs true for an instantiation $\featuresvalue, \predictionvalue$ iff $\predictionvalue = \df(\featuresvalue)$.  For instance, if $\feature_i$ is binary, we have separate variables $\ind_{\featurevalue_i}, \ind_{\bar{\featurevalue_i}}$ which take the place of $\feature_i$ and $\neg \feature_i$ respectively in the circuit.
With the inputs expressed as indicators, we then introduce an additional binary variable for every internal node in the Boolean circuit (forming "intermediate variables" $\interVars$), and use the Tseitin transformation to convert to CNF. The resulting CNF encoding $\cnfencoding(\allInds_{\features}, \allInds_{\prediction}, \interVars)$ has the property that, for a given instantiation $\featuresvalue$ of the features, the CNF encoding has exactly one model (satisfying assignment) with corresponding feature indicators $\allInds_{\features}$, and that model has $\ind_{\prediction = \df(\featuresvalue)}$ true and all other prediction indicators false. 

We can encode the Bayesian network $\bn$ using the following types of CNF clauses, following \citep{Darwiche02Belief}:
\begin{itemize}
    \item \textit{Indicator Clauses}: For variable $\var \in \vars$ with domain $\varvalue_1, ... \varvalue_k$, we include the following clauses:
    \begin{itemize}
        \item $\ind_{\varvalue_1} \vee ... \vee \ind_{\varvalue_k}$
        \item $\neg \ind_{\varvalue_i} \vee \neg \ind_{\varvalue_j}$ for $i \neq j$
    \end{itemize}
    \item \textit{Parameter Clauses}: For each parameter $\param_{\varvalue|\parentsvalue}$, we include the following clause:
    \begin{itemize}
    \item $\ind_{\varvalue} \wedge \ind_{\parentvalue_1} \wedge ... \wedge \ind_{\parentvalue_d} \implies \param_{\varvalue|\parentsvalue}$
    \item $\ind_{\varvalue} \wedge \ind_{\parentvalue_1} \wedge ... \wedge \ind_{\parentvalue_d} \impliedby \param_{\varvalue|\parentsvalue}$
    \end{itemize}
\end{itemize}

The models (satisfying assignments) of the resulting encoding $\cnfencoding_{\bn}(\allInds_{\vars}, \allParam)$ each correspond to an instantiation $\varsvalue$ of the variables $\vars$, with parameter variables set to true or false depending on whether they apply to that instantiation. Thus, they directly correspond to the terms of the network polynomial $l_{\bn}[\allInds_{\vars}, \allParam]$.

We now consider the joint encoding $\cnfencoding_{joint} = \cnfencoding_{\bn} \wedge \cnfencoding_{\df}$, on variables $\allInds := \allInds_{\vars} \cup \allInds_{\prediction}, \allParam, \interVars$:

\begin{restatable}{lemma}{propJoint} \label{prop:jointmodels}
The models of the joint encoding $\cnfencoding_{joint} = \cnfencoding_{\bn} \wedge \cnfencoding_{\df}$ correspond to the terms of the network polynomial $l_{\bn_{\augbn}}[\allInds, \allParam]$ (ignoring intermediate variables $\interVars$).
\end{restatable}

\begin{proof}
Since the mechanism for $\prediction$ in $\bn_{\augbn}$ is deterministic, we can forgo parameters for this mechanism, and write the network polynomial as:

$$l_{\bn_{\augbn}}[\allInds, \allParam] = \sum_{\varvalue_1, ..., \varvalue_n} \lambda_{\prediction =\df(\featuresvalue)} \prod_{i = 1}^{n} \ind_{\varvalue_i} \param_{\varvalue_i|\parentsvalue_i}$$

If we conjoin the encodings $\cnfencoding_{\df}$ and $\cnfencoding_{\bn}$, then since each model of $\cnfencoding_{\bn}$ corresponds to an instantiation $\varsvalue$ (including features $\featuresvalue$), the models of the joint encoding are precisely the models of $\cnfencoding_{\bn}$, with the prediction indicator $\ind_{\prediction = \df(\featuresvalue)}$ true (and all other prediction indicators false), and the intermediate variables $\interVars$ taking some values. Thus, each model of the joint encoding directly corresponds to some valid instantiation $\featuresvalue, \predictionvalue$, i.e. a term in the network polynomial.
\end{proof}

Given $\cnfencoding_{joint}$, the next step is to compile the CNF into an equivalent sd-DNNF, a rooted DAG with literals (a variable or its negation) as leaves, and conjunctions/disjunctions as internal nodes, satisfying the following properties:
\begin{itemize}
    \item \underline{Decomposability:} From every $\wedge$-node, no two branches can share a variable.
    \item \underline{Smoothness:} From every $\vee$-node, every branch must contain the same variables. 
    \item \underline{Determinism:} From every $\vee$-node, every two branches must contradict each other.
\end{itemize}

We convert the sd-DNNF to an AC by replacing conjunctions with $*$, disjunctions with $+$, and replacing all negative literals and literals corresponding to intermediate variables $\interVars$ with the value $1$. 

\thmJoint*

\begin{proof}
The logic can be expressed as a disjunction of the complete subcircuits of the d-DNNF (complete subcircuits are generated by traversing the circuit, choosing all children of every $\wedge$-node and one child of every $\vee$-node). By decomposability and smoothness, each complete subcircuit is a conjunction which specifies a value (and only one value) for each CNF variable $\allInds, \allParam, \interVars$. By Proposition \ref{prop:jointmodels}, each complete subcircuit must thus correspond to a term of the network polynomial. Further, no other complete subcircuit can correspond to that term; otherwise determinism is violated. That is, there is a one-to-one correspondence between the complete subcircuits of the d-DNNF, and the terms of $l_{\bn{\augbn}}[\allInds, \allParam]$.

Each complete subcircuit thus corresponds to an instantiation $\varsvalue$, and takes the form

\begin{align*}
&\; \;\; \left(\bigwedge_{i = 1}^{n} \ind_{\varvalue_i} \wedge \param_{\varvalue_i|\parentsvalue_i} \bigwedge_{v_i' \neq v_i} \neg \ind_{\varvalue_i'} \wedge \neg \param_{\varvalue_i|\parentsvalue_i}\right) \\
&\wedge \left(\ind_{\prediction = \df(\featuresvalue)} \bigwedge_{\predictionvalue' \neq \df(\featuresvalue)} \neg \ind_{\prediction  = \predictionvalue'}\right) \\
&\wedge \left(\bigwedge_{\interVar \in \interVars_{\varsvalue}} \interVar \bigwedge_{\interVar \in \interVars \setminus \interVars_{\varsvalue} }\neg \interVar \right) \\
\end{align*}

where $\interVars_{\varsvalue}$ are the intermediate variables which are true for instantiation $\varsvalue$. If we now exchange conjunctions for $\times$, disjunctions for $+$, and set negative literals and $\interVars$-literals to 1, we get that the terms of the complete subcircuits of the AC take the form:

$$\ind_{\prediction = \df(\featuresvalue)} \prod_{i = 1}^{n} \ind_{\varvalue_i} \param_{\var_i|\parentsvalue_i} $$

Thus $l_{\ac}[\allInds, \allParam]$ is equivalent to $l_{\bn_{\augbn}}[\allInds, \allParam]$. For the second part, for any fixed value of the parameters $\allParam$, the network/AC polynomials will be equivalent multi-linear functions of the indicators, and thus computation of marginals on the respective polynomials will yield the same result. 
\end{proof}

\subsection{Elimination Ordering} \label{apx:elim}
The C2D compiler makes use of a \textit{dtree} in the compilation process. One method of generating a dtree is using an elimination ordering $\pi$: that is, an ordering of the CNF variables $\allInds, \allParam, \interVars$. We now specify the constraints on this ordering that need to be imposed, in order to compile an $\ac$ to be used in the upper bounding algorithm (Algorithm \ref{alg:bottom_up}). Each constraint takes the form $c < c'$, which indicates that CNF variable $c$ must come before $c'$ in the ordering.

Firstly, for all $\ind \in \allInds$, $\interVar \in \interVars$, we impose constraints $\interVar < \ind$. This ensures that the AC contains only $+$-nodes associated with indicator variables $\allInds$, and not intermediate variables $\interVars$.

Secondly, for each constraint $(V_j, V_i)$ corresponding to the intervention set, we impose constraints $\ind_{\varvalue_i} < \ind_{\varvalue_j}$ for every value $\varvalue_i$ of $\var_i$ and every value $\varvalue_j$ of $\var_j$. This ensures that the AC satisfies the constraint $(V_j, V_i)$; that is, $+$-nodes splitting on $V_j$ cannot appear after $+$-nodes splitting on $V_i$.

Given these constraints, we find an elimination ordering $\pi$ using the constrained min-fill heuristic. The variables in the ordering are chosen one at a time. At each point, the min-fill heuristic associates a cost with each CNF variable. We pick the lowest-cost variable which would not violate any ordering constraint.

\subsection{Encoding Options}

Several improvements to the encoding in \citep{Darwiche02Belief} have been proposed, which we address in the context of our joint compilations of causal BNs:
\begin{itemize}
    \item \textbf{Dropping parameter clauses:} \citep{Chavira05Compiling} observed that it is possible to drop the second type of parameter clause (reverse implication), which we employ in our experiments. This introduces additional spurious models for the encoding (since the truth values of some parameter variables will be unspecified), which were removed in \citep{Chavira05Compiling} by performing a \textit{minimization} operation on the resulting sd-DNNF. We instead take advantage of the dtree-based compilation approach in the C2D compiler, which does not introduce negative parameter literals into the d-DNNF (since they are not present in the CNF encoding), and forgo smoothing the circuit. This results in a d-DNNF which does not contain negative parameter literals, while satisfying decomposability/smoothness/determinism with respect to the Bayesian network variables. 
    \item \textbf{Determinism and Parameter Equality:} Determinism refers to parameters $\param_{\varvalue|\parentsvalue}$ taking the value $0$ in the BN, and parameter equality refers to two such parameters in the same CPT taking the same value. Both of these can be encoded using simple modifications to the encoding clauses \citep{Chavira05Compiling}, potentially resulting in much smaller compiled ACs by removing excess parameter variables. The problem with using such encodings is that there no longer exists a unique parameter variable corresponding to each parameter in the BN. As such, we cannot compute the effect of interventions on the network which affect those parameters.
    
    However, it is in theory possible to encode determinism and parameter equality only for CPTs which we know in advance we will never want to intervene on, while avoiding other CPTs. For simplicity, we avoid encoding determinism or parameter equality entirely in our experiments.
\end{itemize}

\section{Proof of Proposition \ref{thm:ac_size2}}

\thmSize*

Denote the set of $n$ CNF variables in $\cnfencoding_{\bn}$ as $\bm{B}_{\bn}$, and the set of $n'$ CNF variables in $\cnfencoding_{F}$ as $\bm{B}_{F}$. The only CNF variables shared between $\cnfencoding_{\bn}$ and $\cnfencoding_{F}$ are the indicators $\lambda_{\featurevalue}$ for BN variables $\feature \in \features$.
Thus the joint encoding $\cnfencoding_{joint} =  \cnfencoding_{bn} \wedge \cnfencoding_{\df}$ has $n + n' - |\allInds_{\features}|$ variables.

The treewidth of a CNF formula is defined to be the treewidth of its interaction graph $\graph_{\cnfencoding}$:
\begin{definition}
    The \underline{interaction graph} $\graph_{\cnfencoding}$ of a CNF formula $\cnfencoding$ is a graph with CNF variables as nodes, and edges between each variables appearing in the same CNF clause.
\end{definition}
Let $\pi$ be an ordering of the nodes in $\graph_{\cnfencoding}$, called an \textit{elimination order}. The treewidth of $\pi$ with respect to $\graph_{\cnfencoding}$ is defined by the following procedure. Remove the CNF nodes according to $\pi$, and connect all pairs of nodes connected to a removed node immediately after removing it. Then the treewidth is the maximum number of edges a variable has immediately before it is removed.
The treewidth of $\graph_{\cnfencoding}$ is then defined to be the minimum treewidth among all elimination orders $\pi$. 

Let $\pi_{\bn}^*$ be the optimal such ordering on $\graph_{\cnfencoding_{\bn}}$, which has treewidth $w$. Let $\pi_{\bn - \features}^*$ be that ordering with all CNF variables in $\allInds_{\features}$ removed. Then consider removing nodes from the interaction graph  $\graph_{\cnfencoding_{joint}}$ in that order. At the time of removal of any node $b$ in $\pi_{\bn - \features}^*$, it will not be connected to any node in $\bm{B}_{\bn} \setminus \allInds_{\features}$ in $\graph_{\cnfencoding_{joint}}$ that it was not connected to at the time of its removal in $\pi_{\bn}^*$ in $\graph_{\cnfencoding_{\bn}}$. This is since the interspersed removals of nodes in $\allInds_{\features}$ in the $\pi_{\bn}^*$ ordering can only have added more edges between nodes in $\bm{B}_{\bn} \setminus \allInds_{\features}$. However, in general, $b$ may be connected to any or all of the nodes in $\allInds_{\features}$. Thus the maximal degree of any removed node (at the time of its removal) is at most $w + |\allInds_{\features}|$.

After removing nodes in $\bm{B}_{\bn} \setminus \allInds_{\features}$, only nodes in $\bm{B}_{\df}$ remain. Let $\pi_{\df}^*$ be the optimal ordering on $\graph_{\cnfencoding_{\df}}$ with treewidth $w'$. Then we can remove those nodes according to that ordering, such that the maximal degree of any removed node (at the time of its removal) is at most $w'$.

The treewidth of the ordering $\pi = (\pi_{\bn - \features}^*, \pi_{\df}^*)$ with respect to $\graph_{\cnfencoding_{joint}}$ is at most $\max(w + |\allInds_{\features}|, w')$, and thus the treewidth of $\cnfencoding_{joint}$ is at most this quantity. If we consider an ordering which follows a similar procedure to above, except that we remove nodes $\bm{B}_{\df} \setminus \allInds_{\features}$ first, we get another bound $\max(w, w'+ |\allInds_{\features}|)$. Thus the treewidth of $\cnfencoding_{joint}$ is at most $\min(\max(w + |\allInds_{\features}|, w'), \max(w, w'+ |\allInds_{\features}|)) = \max(w, w', \min(w, w') + |\allInds_{\features}|)$.

\section{Proof of Theorem \ref{thm:ub_correctness}} \label{apx:ub_proof}
\thmUB*
\newcommand{\coeff}{w}
\newcommand{\bucons}{cs_{UB}}
\newcommand{\buval}{val_{UB}}
\newcommand{\fcons}[1]{cs_{#1}}
\newcommand{\fval}[1]{val_{#1}}
\newcommand{\prefval}[1]{val_{pref, #1}}
\newcommand{\upgrade}{upgrade}
In what follows, the intervention set (including intervenable variables $\intvVars$) and evidence $\bm{e}$ is fixed throughout, and thus is dropped from notation for brevity.

\subsection{Definitions}

First, we review the effect of an intervention on the network polynomial of $\bn_{\augbn}$. For $\intvVar_i \in \intvVars$, let $\parents_i$ denote the parents of $\intvVar$ ($\Pa_{\graph}(\intvVar)$/$\Pa_{\graph'}(\intvVar)$ for parametric/\struc intervention sets respectively). Then, any intervention in the set $\intvSet_{\bn_{\augbn}}$ can be specified as a list of functions $\bm{f} = (f_1, ..., f_{|\intvVars|})$, where $f_i$ maps from $\parents_i$ to $\intvVar_i$, and we write $\bn'_{\bm{f}}$ to represent the intervened network. This corresponds to setting, for every instantiation $\parentsvalue_i$ of $\parents_i$:
\begin{itemize}
    \item $\param_{f(\parentsvalue_i)|\parentsvalue_i} := 1$;
    \item $\param_{\intvVarvalue_i|\parentsvalue_i)} := 0$ for any $\intvVarvalue_i \neq f(\parentsvalue_i)$.
\end{itemize}

The network polynomial $l_{\bn'_{\bm{f}}}[\allInds, \allParam]$ of the intervened network $\bn'$ is then given by applying these changes to each term of $l_{\bn_{\augbn}}[\allInds, \allParam]$. This can be achieved by setting all parameters $\param_{f(\parentsvalue_i)|\parentsvalue_i}$ for intervenable variables $\intvVar_i \in \intvVars$ to $1$, then filtering out those which are incompatible with $\bm{f}$. The following definitions of \textit{weight} and \textit{consistency} capture these notions, given evidence $\evidence$:

\begin{definition}
The \underline{$\bm{f}$-consistency} $\fcons{\bm{f}}(\alpha)$ of a term $\alpha$ of $l_{\bn_{\augbn}}[\allInds, \allParam]$ is defined as:
\begin{equation*}
\fcons{\bm{f}}(\alpha) = 
\begin{cases}
		1  & \text{each parameter in $\alpha$ of the form $\theta_{\intvVar|\parents}$} \\
		   & \text{for $\intvVar \in \intvVars$ is assigned $1$ by $\bm{f}$} \\
		0  & \text{otherwise} 
\end{cases}
\end{equation*}
We say that $\alpha$ is \underline{$\bm{f}$-consistent} if $\fcons{\bm{f}}(\alpha) = 1$,  and \underline{$\bm{f}$-inconsistent} otherwise.
\end{definition}

\begin{definition}
The \underline{weight} $\coeff(\alpha)$ of a term $\alpha$ of $l_{\bn_{\augbn}}[\allInds, \allParam]$ is obtained by evaluating $\alpha$ after:
\begin{itemize}
    \item Assigning $1$ to indicators if they are consistent with $\evidence$, and $0$ otherwise;
    \item Setting all parameters $\param_{f(\parentsvalue_i)|\parentsvalue_i}$ to 1 (other parameters are unchanged)
\end{itemize}  
\end{definition}

\begin{definition}
The \underline{$\bm{f}$-value} $\fval{\bm{f}}(\alpha)$ of a term $\alpha$ of $l_{\bn_{\augbn}}[\allInds, \allParam]$ is defined as:
\begin{equation*}
\fval{\bm{f}}(\alpha) = \coeff(\alpha) \times \fcons{\bm{f}}(\alpha)
\end{equation*}

For any set of terms $S$, we also define:
\begin{equation*}
\fval{\bm{f}}(S) = \sum_{\alpha \in S} \fval{\bm{f}}(\alpha)
\end{equation*}
\end{definition}

\begin{proposition}
Let $S()$ be the set of all terms in $l_{\bn_{\augbn}}[\allInds, \allParam]$. Then:
\begin{equation*}
    p_{\bn'_{\bm{f}}}(\evidence) = \fval{\bm{f}}(S())
\end{equation*}
\end{proposition}

This proposition simply captures how marginal probabilities for intervened Bayesian networks are calculated. 

Now we will define similar concepts for the arithmetic circuit $\ac$. First, recall the definition of a \textit{complete subcircuit} of an AC:
\defSubcircuit*

A \textit{partial subcircuit} is a subset of the nodes (and edges from those nodes) chosen by any complete subcircuit. The \textit{term} of a partial subcircuit $\alpha$ is also the product of all leaf nodes in $\alpha$. Often it will be useful to partition a complete subcircuit into a \textit{prefix} $\alpha_P$ and \textit{suffix} $\alpha_S$: that is, two partial subcircuits which partition the nodes of a complete subcircuit, such that all nodes in the prefix are non-descendants of those in the suffix.

Since $\ac$ has the same polynomial as $\bn_{\augbn}$, each complete subcircuit corresponds to a term of the network polynomial, and we can define very similar notions to those above for the network polynomial, but based on the UB algorithm. 

First, the UB algorithm (Algorithm \ref{alg:bottom_up}) effectively discards some subcircuits by maximizing at some $+$-nodes:

\begin{definition}
The \underline{UB-consistency} $\bucons(\alpha)$ of a complete or partial subcircuit $\alpha$ is defined as:
\begin{equation*}
\bucons(\alpha) = 
\begin{cases}
		1  & \text{if for every $+$-node $t$ in $\alpha$ splitting on } \\
		   & \text{$\intvVar \in \intvVars$, $\alpha$ chooses the same branch } \\
		   & \text{from $t$ as the UB algorithm}\\
		0  & \text{otherwise} 
\end{cases}
\end{equation*}
We say that $\alpha$ is \underline{UB-consistent} if $\bucons{\bm{f}}(\alpha) = 1$,  and \underline{UB-inconsistent} otherwise.
\end{definition}

Second, note that the UB algorithm sets parameters for intervenable variables $\intvVar \in \intvVars$ to 1. Thus, for any complete subcircuit $\alpha$ (corresponding to a network polynomial term), in a slight abuse of notation, we can define the weight $\coeff(\alpha)$ in the same way as for network polynomial terms. We can also define the weight for partial subcircuits in the obvious way.

\begin{definition}
The \underline{weight} $\coeff(\alpha)$ of a complete or partial subcircuit $\alpha$ is obtained by evaluating $term(\alpha)$ after:
\begin{itemize}
    \item Assigning $1$ to indicators if they are consistent with $\evidence$, and $0$ otherwise;
    \item Setting all parameters $\param_{f(\parentsvalue_i)|\parentsvalue_i}$ to 1 (other parameters are unchanged)
\end{itemize}  
\end{definition}

\begin{definition}
The \underline{UB-value} $\buval(\alpha)$ of a complete or partial subcircuit $\alpha$ is defined as:
\begin{equation*}
\buval(\alpha) = \coeff(\alpha) \times \bucons(\alpha) 
\end{equation*}

For any set of complete or partial subcircuits $S$, we also define:
\begin{equation*}
\buval(S) = \sum_{\alpha \in S} \buval(\alpha)
\end{equation*}
\end{definition}

Note that the only distinction from $\bm{f}$-value is that we use UB-consistency instead of $\bm{f}$-consistency, i.e. we set different complete subcircuits/terms to 0.

\begin{proposition}
Let $S()$ be the set of all complete subcircuits in the AC. Then:
\begin{equation*}
   UB(\mathcal{AC}, \bm{e}, \bm{W}) = \buval(S())
\end{equation*}
\end{proposition}

This proposition simply states that assigning value $0$ to subcircuits unselected by the UB-algorithm and value $\coeff(\alpha)$ to selected subcircuits produces the output of the UB-algorithm.

Recall that our goal is to show that the UB-algorithm produces an upper bound to $\max_{\bn' \in \intvSet_{\bn{\augbn}}} p_{\bn'}(\evidence) = \max_{\bm{f}} p_{\bn'_{\bm{f}}}(\bm{e})$. By the two Propositions, this is equivalent to showing:
\begin{equation} \label{eq:val_relation}
    \buval(S()) \geq \max_{\bm{f}} \fval{\bm{f}}(S())
\end{equation}

\subsection{Inductive Lemma}

In order to prove Equation \ref{eq:val_relation}, we will need to prove a Lemma, which relies on some additional definitions.

\begin{definition}[Subcircuit sets]
    We will use the following notation to denote various sets of complete and partial subcircuits. Here, $t$ denotes a $+$-node in the AC, and $\bm{t} = \{t_1, ... t_j\}$ denotes a set of $+$-nodes such that \textbf{no node is a descendant of another}.
    \begin{itemize}
        \item $S(\bm{t})$ denotes all complete subcircuits which cross all $+$-nodes in $\bm{t}$;
        \item $S_P(\bm{t})$ denotes all "prefix" partial subcircuits, obtained by selecting all visited nodes which are non-descendants of $t$ from complete subcircuits in $S(\bm{t})$;
        \item $S_S(t)$ denotes all "suffix" partial subcircuits, obtained by traversing the AC top-down starting at $t$, choosing one child of every $+$-node and every child of every $\times$-node.
        \item $S(\bm{t}, \alpha_0)$ denotes a set of complete subcircuits, defined for any $\alpha_0 \in S_P(\bm{t}):$
        $S(\bm{t}, \alpha_0) = \{(\alpha_0, \alpha_1, ..., \alpha_j): \alpha_1 \in S_S(t_1), ..., \alpha_j \in S_S(t_j)\}$
    \end{itemize}
\end{definition}

We note a few important facts regarding these sets. Firstly, $S(\bm{t})$ can be empty, if $\bm{t}$ includes two nodes which never appear in the same complete subcircuit. Second, $S(\bm{t})$ can be represented as the product set $S_P(\bm{t}) \times S_S(t_1) \times ... \times S_S(t_j)$, and $\{S(\textbf{t}, \alpha_0): \alpha_0 \in S_P(\textbf{t})\}$ forms a partition of $S(\bm{t})$. Finally, both the UB-consistency and weight of a complete subcircuit $\alpha = (\alpha_0, \alpha_1, ... \alpha_j) \in S(\bm{t})$ can be decomposed. $\alpha$ is consistent iff its prefix and suffixes are all consistent, that is:
$$\bucons(\alpha) = \bucons(\alpha_0) \prod_{i = 1} ^ {j} \bucons(\alpha_i)$$
In addition, the weight of $\alpha$ is given by:
$$\coeff(\alpha) = \coeff(\alpha_0) \prod_{i = 1} ^ {j} \coeff(\alpha_i)$$ 

We now define a modified version of UB-value for a complete subcircuit $\alpha$ which crosses nodes $\bm{t}$. Informally, this assumes the prefix $\alpha_0$ to be consistent for the purposes of computing value.

\begin{definition}
Given a set of $+$-nodes $\bm{t}$, and a complete subcircuit $\alpha$ crossing $\bm{t}$ which consists of prefix subcircuit $\alpha_0 \in S_P(\bm{t})$ and suffix subcircuits $\alpha_i \in S_S(t_i)$, the \underline{prefix-consistent BU-value} $\prefval{\bm{t}}(\alpha)$ is defined as:
\begin{equation*}
    \prefval{\bm{t}}(\alpha) = \coeff(\alpha) \times \prod_{i = 1}^{j} \bucons(\alpha_i)
\end{equation*}

For any set of such complete subcircuits $S$, 
\begin{equation*}
    \prefval{\bm{t}}(S) = \sum_{\alpha \in S} \prefval{\bm{t}}(\alpha)
\end{equation*}
\end{definition}

Now, we can state our main result:
\begin{lemma} \label{lem:thm_redux}
Let $\bm{t}$ be a set of $+$-nodes, such that no $t \in \bm{t}$ is a descendant of another. Further, let $\alpha_0 \in S_P(\bm{t})$. Then:
\begin{equation} \label{eq:lem_ineq}
\prefval{\bm{t}}(S(\bm{t}, \alpha_0)) \geq \max_{\bm{f}} \fval{\bm{f}}(S(\bm{t}, \alpha_0))
\end{equation}
\end{lemma}

As previously stated, for many sets $\bm{t}$, $S(\bm{t}, \alpha_0)$ may be empty; in such cases, the Lemma is trivially true.

Before proving this Lemma, we demonstrate how this can be used to derive the inequality in Equation \ref{eq:val_relation}, and thus, the Theorem. Take $\bm{t}$ to be the singleton set $\{r\}$, where $r$ is the root $+$/$MAX$-node of the AC. Then the (only) prefix $\alpha_0$ of $\bm{t}$ is empty, so that $\prefval{\bm{t}}(S(\bm{t}, \alpha_0)) = \buval(S(\bm{t}, \alpha_0))$. Further, $S(\bm{t}, \alpha_0) = S()$ consists of all subcircuits in the AC. Thus, in this case, the Lemma reduces to Equation \ref{eq:val_relation}.

\begin{proof}[Proof of Lemma \ref{lem:thm_redux}]
We will prove the Lemma by induction. Let $\bm{\pi} = \{\pi_1, ..., \pi_n\}$ be a reverse topological ordering of the $n$ $+$-nodes in the AC (that is, descendants come before ancestors), and we write $t < t'$ to indicate a node $t$ comes before $t'$ in the ordering. In particular, our inductive hypothesis at step $k$ will be that the lemma holds for all subsets $\bm{t}$ containing only nodes in $\bm{\pi}_{\leq k} = \{\pi_1, ... \pi_k\}$. 

At step $k$ (for $1 \leq k \leq n$), we need to show the inequality in Equation \ref{eq:lem_ineq} for any set of $+$-nodes $\bm{t}' = \{t_1', ..., t_{j - 1}', t_j'\}$ s.t. $t_j' = \pi_k$, $0 \leq j \leq k$, and $t_i' \in \bm{\pi}_{\leq k} \;\; \forall 1 \leq i \leq j - 1$ (so that $t_j' := \pi_k$), and for any $\alpha_0 \in S_P(\bm{t}')$. The proof for step $k$ follows one of 2 templates, depending on the node $\pi_k$, in particular, whether it splits on a variable in $\bm{W}$.

\subsubsection{Does not split on \texorpdfstring{$\bm{W}$}{W}}
    
    Suppose that $\pi_k$ has $m$ branches. For each branch $1 \leq b \leq m$, starting from the prefix $\alpha_0$, add $\pi_k$, and then traverse down branch $b$, adding all children of each $\times$-node, until we reach a leaf node (included) or a $+$-node (excluded). Denote this extended prefix subcircuit $\alpha^{(b)}_0 = (\alpha_0, \alpha^{(b)})$ (where $\alpha^{(b)}$ is defined to be the partial subcircuit added), and the $+$-nodes reached by $\bm{t}^{(b)} = \{t^{(b)}_1, ... t^{(b)}_l\}$ (possibly empty). Then define $\bm{t'_{-j}} = \{t'_1, ... t'_{j - 1}\}$, that is, $\bm{t'}$ with $\pi_k$ removed, and define $\bm{t}'^{(b)} = (\bm{t}'_{-j}, \bm{t}^{(b)})$. This fulfils the condition that no two nodes are descendants of each other, since $\bm{t}^{(b)}$ are immediate descendants of $\pi_k$. 
    
    Every complete subcircuit $\alpha \in S(\bm{t}'^{(b)})$ is then of the form $(\alpha_0^{(b)}, \alpha_1, ... \alpha_{j - 1}, \alpha^{(b)}_1, ... \alpha^{(b)}_l)$, with prefix $\alpha_0^{(b)} \in S_P(\bm{t}'^{(b)})$, and suffixes $\alpha_i \in S_S(t'_i)$ ($1 \leq i \leq j - 1$) and $\alpha^{(b)}_a \in S_S(t^{(b)}_a)$ ($1 \leq a \leq l$). 
    
    Recall that $S(\bm{t}'^{(b)}, \alpha_0^{(b)})$ denotes the set of all complete subcircuits with prefix $\alpha_0^{(b)}$ and any combination of suffix subcircuits. The prefix-consistent BU-value for these subcircuits is greater than the $\bm{f}$-value for any $\bm{f}$:
    \begin{align*}
        &\prefval{\bm{t}'^{(b)}}(S(\bm{t}'^{(b)}, \alpha_0^{(b)})) \\
        &\geq \max_{\bm{f}} \fval{\bm{f}}(S(\bm{t}'^{(b)}, \alpha_0^{(b)}))
    \end{align*}
    This follows directly from the inductive hypothesis for step $k-1$. This is because $\bm{t}'_{-j}$ only contains nodes from $\bm{\pi}_{\leq k - 1}$ by definition, and $\bm{t}^{(b)}$ are descendants of $\pi_k$, so by the fact $\pi$ is a reverse topological order, they too are taken from $\bm{\pi}_{\leq k - 1}$.
    
    Further, $\{S(\bm{t}'^{(b)}, \alpha_0^{(b)}): 1 \leq b \leq m\}$ forms a partition of $S(\bm{t}', \alpha_0)$ (each set consists of the set of subcircuits following a particular branch from $\pi_k$). 
    
    Now consider the UB-consistency of the suffix subcircuit from $t_j'$, $\alpha_j = (\alpha^{(b)}, \alpha^{(b)}_1, ... ,\alpha^{(b)}_l)$. We have $\bucons(\alpha_j) = \bucons(\alpha^{(b)}) \times \prod_{a = 1}^{l} \bucons(\alpha^{(b)}_a) = \prod_{a = 1}^{l} \bucons(\alpha^{(b)}_a)$. The final equality follows from the fact that $t_j' = \pi_k$ does not split on $\intvVar \in \intvVars$ and is the only $+$-node in $\alpha^{(b)}$, so $\alpha^{(b)}$ contains no $+$-nodes splitting on $\intvVar \in \intvVars$, and it follows that $\bucons(\alpha^{(b)}) = 1$.
     
     We use this to combine the results for each branch:
    {
    \begin{align*}
        &\prefval{\bm{t}'}(S(\bm{t}', \alpha_0)) \\
        &= \sum_{\alpha \in S(\bm{t}', \alpha_0)} \coeff(\alpha) \times \prod_{i = 1}^{j} \bucons(\alpha_i) \\
        &= \sum_{b = 1}^{m} \sum_{\alpha \in S(\bm{t}'^{(b)}, \alpha_0^{(b)})} \coeff(\alpha) \times \prod_{i = 1}^{j} \bucons(\alpha_i) \\
        &= \sum_{b = 1}^{m} \sum_{\alpha \in S(\bm{t}'^{(b)}, \alpha_0^{(b)})} \coeff(\alpha) \times \prod_{i = 1}^{j - 1} \bucons(\alpha_i) \\
        & \hspace{2.7cm} \times \prod_{a = 1}^{l} \bucons(\alpha^{(b)}_a)\\
        &= \sum_{b = 1}^{m} \prefval{\bm{t}'^{(b)}}(S(\bm{t}'^{(b)}, \alpha_0^{(b)})) \\
        &\geq \sum_{b = 1}^{m} \max_{\bm{f}} \fval{\bm{f}}(S(\bm{t}'^{(b)}, \alpha_0^{(b)})) \\
        &\geq \max_{\bm{f}} \sum_{b = 1}^{m} \fval{\bm{f}}(S(\bm{t}'^{(b)}, \alpha_0^{(b)})) \\
        &= \max_{\bm{f}} \fval{\bm{f}}(S(\bm{t}', \alpha_0))
    \end{align*}
    }%
    
    The second equality uses the partition, the third is by the facts about consistency shown above, the fourth equality is by definition, the fifth inequality is by inductive hypothesis, the sixth is a standard sum/max swap, and the final equality is again by the partition.
    
\subsubsection{Does split on \texorpdfstring{$\bm{W}$}{W}}
    
    In the case where $\pi_k$ splits on some $W \in \bm{W}$, i.e. is a $MAX$-node, the above template doesn't work because the BU algorithm now chooses one of the branches, rather than adding all branches together. 

    For a given intervention $\bm{f}$, denote the subset of $\bm{f}$-consistent complete subcircuits crossing $\bm{t}'$ by  $S_{\bm{f}}(\bm{t}', \alpha_0)$. In this context, we will consider such a complete subcircuit $\alpha$ to be split into a prefix consisting of $\alpha_{-j} = (\alpha_0, \alpha_1, ... \alpha_{j - 1}$ with $\alpha_0 \in S_P(\bm{t}')$, $\alpha_i \in S_S(t_i')$, and a single suffix $\alpha_j \in S_S(t_j')$. The following two defintions are useful:
    \begin{itemize}
        \item First, we define the set of "$\bm{f}$-consistent prefixes":
        \begin{equation}
        S_{-j, \bm{f}}(\bm{t}') = \{\alpha_{-j}: \exists \alpha_j \text{ s.t. } \fcons{\bm{f}}(\alpha) = 1\}
        \end{equation}
        \item Second, for a given prefix $\alpha_{-j}$, define the set of complete subcircuits $S_{\bm{f}}(\pi_k, \alpha_{-j})$ to be the subset of $S(\pi_k, \alpha_{-j})$ which is $\bm{f}$-consistent. This is empty if the prefix $\alpha_{-j}$ is not consistent.
    \end{itemize}
    
    Our strategy will be to show that we can modify $\bm{f}$ so that all $\bm{f}$-consistent subcircuits in $S(\bm{t}', \alpha_0)$ include the same branch of $\pi_k$, without reducing the $\bm{f}$-value. We can then use our inductive hypothesis to bound this value. First, we show that given $\alpha_{-j}$, any resulting complete subcircuit must go down the same branch:
    
    \begin{lemma} \label{lem:branches}
    For a given $\bm{f}$ and $\alpha_{-j}$, subcircuits $S_{\bm{f}}(\pi_k, \alpha_{-j})$ all choose the same branch at $\pi_k$.
    \end{lemma}
    
    \begin{proof}
    Here we use our crucial assumption of the AC ordering constraints, namely, that no node that is a descendant of $\pi_k$ splits on any variable in $\bm{U} = pa(W)$. Since subcircuits in $S(\pi_k, \alpha_{-j})$ differ only on their suffix from $\pi_k$, they all share the same indicators for variables in $pa(W)$, corresponding to an instantiation $\bm{u}$. Suppose $f_W(\bm{u}) = w$. Then, any parameter $\theta_{W = w'|\bm{U} = \bm{u}}$ for $w' \neq w$ is intervened to 0. Since $\pi_k$ splits on $W$, all branches from $\pi_k$ have different indicators for $W$. Thus, only the branch with indicator corresponding to $W = w$ can have $\bm{f}$-consistent subcircuits.
    \end{proof}
    
    This does not preclude, however, that there may be $\bm{f}$-consistent subcircuits with different prefixes $\alpha_{-j}, \beta_{-j}$ which choose different branches. The following Lemma shows that we can find another intervention, $\bm{f}_{adj}$, that does always (effectively) choose the same branch for any prefix, and additionally has greater or equal value.
    
    \begin{lemma} \label{lem:greater_intv}
    For any intervention $\bm{f}$, there exists another intervention $\bm{f}_{adj}$ such that:
    \begin{itemize}
        \item $\fval{\bm{f}_{adj}}(S(\bm{t}', \alpha_0)) \geq \fval{\bm{f}}(S(\bm{t}', \alpha_0))$
        \item All subcircuits $\alpha \in S_{\bm{f}_{adj}}(\bm{t}', \alpha_0)$ with non-zero weight $\coeff(\alpha)$ choose the same branch at $\pi_k$ (call it $b_{adj}$)
    \end{itemize}
    \end{lemma}
    
    \begin{proof}
    
    We will define a non-constructive operation $\texttt{upgrade}$, which takes as input an intervention $\bm{f}$, and outputs an intervention $\bm{f}_{up}$.
    
    Consider all subcircuits $\alpha = (\alpha_0, \alpha_1, ..., \alpha_{j - 1}, \alpha_j) \in S_{\bm{f}}(\pi_k, \alpha_{-j})$. Define $S_{j, \bm{f}}(\bm{t}', \alpha_{-j}) \subseteq S_S(t'_j)$ as follows:
    \begin{equation*}
        S_{j, \bm{f}}(\bm{t}', \alpha_{-j}) = \{\alpha_j \in S_S(t_j'): \alpha \in S_{\bm{f}}(\pi_k, \alpha_{-j}) \}
    \end{equation*}
    That, is for a given prefix $\alpha_{-j}$ and $\bm{f}$, this is the set of suffix partial subcircuits from $\pi_k$ which are $\bm{f}$-consistent. 
    
    The "optimal set" of suffixes over $\bm{f}$ and $\alpha_j$ is defined to be that which attains the greatest combined weight, that is, $\max_{\bm{f}} \max_{\alpha_{-j}} \coeff(S_{j, \bm{f}}(\bm{t}', \alpha_{-j}))$. In general the optimal set will not be unique; we use $\bm{S}_j^{*}$ to denote the set of optimal suffix sets. We pick (arbitrarily) some $S_j^{*} \in \bm{S}_j^{*}$, and denote the intervention $\bm{f}$ which attains this $\bm{f}^{*}$, and the prefix which attains this $\alpha_{-j}^{*}$.
    
     Now, we will define a set of complete subcircuits $T_{\bm{f}}(\bm{t}', \alpha_0)$ by taking all $\bm{f}$-consistent prefixes $\alpha_{-j}$, and taking the product set with the set of suffixes $S_j^{*}$. Formally:
    \begin{equation}
        T_{\bm{f}}(\bm{t}', \alpha_0) = \{(\alpha_{-j}, \alpha_j): \alpha_{-j} \in S_{-j, \bm{f}}(\bm{t}'), \alpha_j \in S_j^{*}\}
    \end{equation}
    
    Intuitvely, we are "increasing the weight" of the $\bm{f}$-consistent subcircuits of $S(\bm{t}', \alpha_0)$, while also ensuring consistent subcircuits all choose the same branch (that chosen by $S_j^{*}$). However, we cannot be sure that $T_{\bm{f}}(\bm{t}', \alpha_0)$ actually corresponds to the $\bm{f}_{up}$-consistent subcircuits of any allowed intervention $\bm{f}_{up}$. We will now show that there exists a $\bm{f}_{up}$ such that all of these subcircuits are $\bm{f}_{up}$-consistent. That is, $T_{\bm{f}}(\bm{t}', \alpha_0) \subseteq  S_{\bm{f}_{up}}(\bm{t}', \alpha_0)$.
    
   Suppose for contradiction this was not the case. Then there must exist two subcircuits $\alpha, \beta \in T_{\bm{f}}(\bm{t}', \alpha_0)$ which are contradictory, that is, for some $W \in \bm{W}$ and its parents $\bm{U} = pa(W)$, $\alpha, \beta$ must contain indicators for the same value of $\bm{U} = \bm{u}$, but indicators $\lambda_{W = w_\alpha}, \lambda_{W = w_\beta}$ for different values of $W$ respectively (that is, $w_\alpha \neq w_\beta$). 
    
    We can write $\alpha = (\alpha_{-j}, \alpha_j)$ and $\beta = (\beta_{-j}, \beta_j)$. Now, we consider three cases according to where the contradictory indicators reside:
    \begin{itemize}
        \item \textbf{If $\lambda_{W = w_\alpha} \in \alpha_{-j}, \lambda_{W = w_\beta} \in \beta_{-j}$}
        
        Recall that by definition of $T_{\bm{f}}(\bm{t}', \alpha_0)$, $\alpha_{-j}, \beta_{-j} \in S_{-j, \bm{f}}(\bm{t}')$, that is, the set of $\bm{f}$-consistent prefixes. Then there must exist $\alpha_j', \beta_j' \in S(t_j')$ such that $\alpha' = (\alpha_{-j}, \alpha_j')$ and $\beta' = (\beta_{-j}, \beta_j')$ are both $\bm{f}$-consistent. But this is a contradiction as two complete subcircuits with conflicting indicators $\lambda_{W = w_\alpha}, \lambda_{W = w_\beta}$ cannot simultaneously be consistent for any $\bm{f}$.
        \item \textbf{If $\lambda_{W = w_\alpha} \in \alpha_{j}, \lambda_{W = w_\beta} \in \beta_{j}$}
        
        Recall that by definition of $T_{\bm{f}}(\bm{t}', \alpha_0)$, $\alpha_j, \beta_j \in S_j^{*}$, that is, they both belong to $S_{j, \bm{f}^{*}}(\bm{t}', \alpha_{-j}^{*})$. In other words, $\alpha^{*} = (\alpha_{-j}^*, \alpha_j)$ and $\beta^{*} = (\alpha_{-j}^{*}, \beta_j)$ are both $\bm{f}^{*}$-consistent. But again this is a contradiction, as they contain conflicting indicators $\lambda_{W = w_\alpha}, \lambda_{W = w_\beta}$.
        
        \item \textbf{Otherwise}
        
        If neither of the above cases hold, then either $\lambda_{W = w_\alpha} \in \alpha_{-j}, \lambda_{W = w_\beta} \in \beta_{j}$ or $\lambda_{W = w_\alpha} \in \alpha_{j}, \lambda_{W = w_\beta} \in \beta_{-j}$. Without loss of generality, let us assume the former. Then consider the subcircuit $\alpha' = (\alpha_{-j}, \beta_j)$. This single complete subcircuit contains both indicators $\lambda_{W = w_\alpha}, \lambda_{W = w_\beta}$, which is a contradiction as no term/complete subcircuit can do so.
    \end{itemize}
    
    Thus, there exists some $\bm{f}_{up}$ such that $T_{\bm{f}}(\bm{t}', \alpha_0) \subseteq  S_{\bm{f}_{up}}(\bm{t}', \alpha_0)$. Our operation $\texttt{upgrade}$ outputs any such $\bm{f}_{up}$.
    $\bm{f}_{up}$ has three important properties:
    \begin{enumerate}
        \item \textbf{Does not remove consistent prefixes} 
        
        $S_{-j, \bm{f}_{up}}(\bm{t'}) \supseteq S_{-j, \bm{f}}(\bm{t'})$
        
        That is, the set of $\bm{f}_{up}$-consistent prefixes subsumes the set of $\bm{f}$-consistent prefixes. This follows from the fact that $T_{\bm{f}}(\bm{t}', \alpha_0) \subseteq  S_{\bm{f}_{up}}(\bm{t}', \alpha_0)$: $T_{\bm{f}}(\bm{t}', \alpha_0)$ contains subcircuits with all prefixes in $S_{-j, \bm{f}}(\bm{t'})$.
        \item \textbf{All suffixes for previous consistent prefixes are in $S_j^{*}$ or have zero weight}
        
         $\forall \alpha_{-j} \in S_{-j, \bm{f}}(\bm{t}'), \forall \alpha_j \in S_{j, \bm{f}_{up}}(\bm{t}', \alpha_{-j})$, either $\coeff(\alpha_j) = 0$ or $\alpha_j \in \bm{S}_j^*$.

         That is, for all $\bm{f}$-consistent prefixes $\alpha_{-j}$, the corresponding $\bm{f}_{up}$-consistent subcircuits $S_{j, \bm{f}_{up}}(\bm{t}', \alpha_{-j})$ are either contained in the optimal suffix set $\bm{S}_j^*$ or have weight 0.
         
         This is shown as follows. From $T_{\bm{f}}(\bm{t}', \alpha_0) \subseteq  S_{\bm{f}_{up}}(\bm{t}', \alpha_0)$, we know that $S_j^* \subseteq S_{j, \bm{f}_{up}}(\bm{t}', \alpha_{-j})$ for any $\alpha_{-j} \in S_{-j, \bm{f}}(\bm{t}')$ (i.e. $\bm{f}$-consistent prefixes). If $S_{j, \bm{f}_{up}}(\bm{t}', \alpha_{-j})$ contains some $\alpha_j \notin S_j^*$, then it must have weight $0$, otherwise $\coeff(S_{j, \bm{f}_{up}}(\bm{t}', \alpha_{-j})) \geq \coeff(S_j^*)$ which is a contradiction as $S_j^*$ is an optimal suffix set.
        \item \textbf{Does not decrease value}
        
        $\fval{\bm{f}_{up}}(S(\bm{t}, \alpha_0)) \geq \fval{\bm{f}}(S(\bm{t}, \alpha_0))$
        
        This is proven as follows:
    \begin{align*}
        &\fval{\bm{f}_{up}}(S(\bm{t}, \alpha_0)) \\
        &= \sum_{\alpha \in S(\bm{t}, \alpha_0)} \coeff(\alpha) \times \fcons{\bm{f}_{up}}(\alpha) \\
        &= \sum_{\alpha \in S_{\bm{f}_{up}}(\bm{t}, \alpha_0)} \coeff(\alpha) \\
        &\geq \sum_{\alpha \in T_{\bm{f}}(\bm{t}, \alpha_0)} \coeff(\alpha) \\
        &= \sum_{\alpha_{-j} \in S_{-j, \bm{f}}(\bm{t}')} \sum_{\alpha_j \in S_j^{*}} \coeff(\alpha) \\
        &= \sum_{\alpha_{-j} \in S_{-j, \bm{f}}(\bm{t}')} \sum_{\alpha_j \in S_j^{*}} \coeff(\alpha_{-j}) \times \coeff(\alpha_j) \\
        &\geq \sum_{\alpha_{-j} \in S_{-j, \bm{f}}(\bm{t}')} \sum_{\alpha \in S_{\bm{f}}(\pi_k, \alpha_{-j})} \coeff(\alpha_{-j}) \times \coeff(\alpha_j) \\
        &= \sum_{\alpha \in S_{\bm{f}}(\bm{t}', \alpha_0)} \coeff(\alpha_{-j}) \times \coeff(\alpha_j) \\
        &= \sum_{\alpha \in S_{\bm{f}}(\bm{t}', \alpha_0)} \coeff(\alpha) \\
        &= \sum_{\alpha \in S(\bm{t}', \alpha_0)} \coeff(\alpha) \times \fcons{\bm{f}}(\alpha) \\
        &= \fval{\bm{f}}(S(\bm{t}, \alpha_0))
    \end{align*}
    \end{enumerate}
    
    Consider applying this operation iteratively, producing a sequence of interventions $\bm{f}_{up}^{(1)}, \bm{f}_{up}^{(3)}, ...$. For each intervention $\bm{f}_{up}^{(i)}$, we can associate two quantities: $m_1^{(i)}$, the number of $\bm{f}_{up}$-consistent prefixes, and among these prefixes, the number $m_2^{(i)}$ ($< m_1^{(i)}$) of prefixes which have corresponding suffixes all in $S_j^{*}$ or having zero weight. Property 1 then tells us that $m_1^{(i + 1)} \geq m_1^{(i)}$, and Property 2 tells us that $m_2^{(i + 1)} \geq m_1^{(i)}$. Since there are a finite number of possible prefixes, we must (in finite time) obtain some $\bm{f}_{up}^{i}$ with $m_1^{(i)} = m_2^{(i)}$, which we will call $\bm{f}_{adj}$. Since all $\bm{f}_{adj}$-consistent prefixes have corresponding suffixes in $S_j^*$ or with weight zero, by Lemma \ref{lem:branches} this intervention indeed has the property that all subcircuits take the same branch or have weight zero, and further by Property 3 the $\bm{f}_{adj}$-value of $S(\bm{t}, \alpha_0)$ is at least $\bm{f}$-value of $S(\bm{t}, \alpha_0)$.
    \end{proof}
    
    The significance of this Lemma is that we can now make use of the inductive hypothesis. As in the case where $\pi_k$ doesn't split on $\intvVar \in \intvVars$, we have $\prefval{\bm{t}'^{(b)}}(S(\bm{t}'^{(b)}, \alpha_0^{(b)})) \geq \fval{\bm{f}_{adj}}(S(\bm{t}'^{(b)}, \alpha_0^{(b)}))$ for any branch $b$. However, since all non-zero weight $\bm{f}_{adj}$-consistent subcircuits follow the same branch,  $\prefval{\bm{t}'^{(b_{adj})}}(S(\bm{t}'^{(b_{adj})}, \alpha_0^{(b_{adj})}))$ is in fact an upper bound on the $\bm{f}_{adj}$-value summed over subcircuits from all branches: $\fval{\bm{f}}(S(\bm{t}', \alpha_0))$.
    
    Now we use this to prove the inductive hypothesis of Lemma \ref{lem:thm_redux} for step $k$. Consider, for each branch $b$ of $\pi_k$, the set of partial subcircuits given by $S_S^{(b)}(\pi_k) := \{(\alpha^{(b)}, \alpha_1^{(b)}, ... \alpha_l^{(b)}): \alpha_a^{(b)} \in S_S(t_a^{(b)})\}$. The UB algorithm chooses the branch with the greatest UB-value \textbf{ignoring consistency of $\alpha^{(b)}$} (i.e. which branch is actually chosen), that is, 
    \begin{equation*}
        v_b := \coeff(\alpha^{(b)}) \times \sum_{\alpha^{(b)}_1, ..., \alpha^{(b)}_l} \prod_{a = 1}^{l} \coeff(\alpha^{(b)}_a) \times \bucons(\alpha^{(b)}_a)
    \end{equation*}  
    Suppose that the UB algorithm chooses branch $b_{UB}$ at $\pi_k$, so that $\bucons(\alpha^{(b_{UB})}) = 1$ and $\bucons(\alpha^{(b)}) = 0$ for all $b \neq b_{UB}$. Then, for a given $\bm{f}$,  we have that:

    \begin{align*}
        &\prefval{\bm{t}'}(S(\bm{t}', \alpha_0)) \\
        &= \sum_{\alpha \in S(\bm{t}', \alpha_0)} \coeff(\alpha) \times \prod_{i = 1}^{j} \bucons(\alpha_i) \\
        &= \sum_{b = 1}^{m} \sum_{\alpha \in S(\bm{t}'^{(b)}, \alpha_0^{(b)})} \coeff(\alpha) \times \prod_{i = 1}^{j} \bucons(\alpha_i) \\
        &= \sum_{b = 1}^{m} \sum_{\alpha \in S(\bm{t}'^{(b)}, \alpha_0^{(b)})} \coeff(\alpha) \times \prod_{i = 1}^{j - 1} \bucons(\alpha_i) \\
        & \hspace{1cm} \times \bucons(\alpha^{(b)}) \times \prod_{a = 1}^{l} \bucons(\alpha^{(b)}_a)\\
        &= \sum_{b = 1}^{m} \sum_{\alpha \in S(\bm{t}'^{(b)}, \alpha_0^{(b)})} \coeff(\alpha_{-j}) \times \prod_{i = 1}^{j - 1} \bucons(\alpha_i) \\
        & \hspace{1cm} \times \bucons(\alpha^{(b)}) \times \coeff(\alpha^{(b)}) \\
        & \hspace{1cm} \times \prod_{a = 1}^{l} \coeff(\alpha^{(b)}_a) \times \bucons(\alpha^{(b)}_a)\\
        &= \sum_{\alpha_1, ..., \alpha_{j - 1}} \coeff(\alpha_{-j}) \times \prod_{i = 1}^{j - 1} \bucons(\alpha_i) \\
        & \hspace{1cm} \times \sum_{b = 1}^{m} \biggl(\bucons(\alpha^{(b)}) \times \coeff(\alpha^{(b)}) \\
        & \hspace{1cm} \times \sum_{\alpha^{(b)}_1, ..., \alpha^{(b)}_l} \prod_{a = 1}^{l} \coeff(\alpha^{(b)}_a) \times \bucons(\alpha^{(b)}_a)\biggr)\\
        &= \sum_{\alpha_1, ..., \alpha_{j - 1}} \coeff(\alpha_{-j}) \times \prod_{i = 1}^{j - 1} \bucons(\alpha_i) \\
        & \hspace{1cm} \times \sum_{b = 1}^{m} v_b \times \bucons(\alpha^{(b)}) \\
        &= \sum_{\alpha_1, ..., \alpha_{j - 1}} \coeff(\alpha_{-j}) \times \prod_{i = 1}^{j - 1} \bucons(\alpha_i) \\
        & \hspace{1cm} \times v_{b_{BU}} \\
        &\geq \sum_{\alpha_1, ..., \alpha_{j - 1}} \coeff(\alpha_{-j}) \times \prod_{i = 1}^{j - 1} \bucons(\alpha_i) \\
        & \hspace{1cm} \times v_{b_{adj}} \\
        &= \sum_{\alpha_1, ..., \alpha_{j - 1}} \coeff(\alpha_{-j}) \times \prod_{i = 1}^{j - 1} \bucons(\alpha_i) \\
        & \hspace{1cm} \times \coeff(\alpha^{(b_{adj})}) \times \\
        & \hspace{1cm} \sum_{\alpha^{(b_{adj})}_1, ..., \alpha^{(b_{adj})}_l} \prod_{a = 1}^{l} \coeff(\alpha^{(b_{adj})}_a) \times \bucons(\alpha^{(b_{adj})}_a)\\
        &= \sum_{\alpha \in S(\bm{t}'^{(b_{adj})}, \alpha_0^{(b_{adj})})} \coeff(\alpha) \times \prod_{i = 1}^{j - 1} \bucons(\alpha_i) \\
        & \hspace{1cm} \times \prod_{a = 1}^{l} \bucons(\alpha^{(b_{adj})}_a)\\
        &= \prefval{\bm{t}'^{(b_{adj})}}(S(\bm{t}'^{(b_{adj})}, \alpha_0^{(b_{adj})}))\\
        &\geq \fval{\bm{f}_{adj}}(S(\bm{t}'^{(b_{adj})}, \alpha_0^{(b_{adj})})) \\
        &= \sum_{\alpha \in S(\bm{t}'^{(b_{adj})}, \alpha_0^{(b_{adj})})} \coeff(\alpha) \times \fcons{\bm{f}_{adj}}(\alpha)\\
        &=\sum_{\alpha \in S(\bm{t}', \alpha_0)} \coeff(\alpha) \times \fcons{\bm{f}_{adj}}(\alpha)\\
        &= \fval{\bm{f}_{adj}}(S(\bm{t}', \alpha_0)) \\
        &\geq \fval{\bm{f}}(S(\bm{t}', \alpha_0)) \\
    \end{align*}

Thus, in both cases (whether $\pi_k$ splits on $W \in \bm{W}$), we have shown that the inductive hypothesis holds at step $k$. Thus Lemma \ref{lem:thm_redux} is proved.
\end{proof}
As noted previously, this Lemma immediately implies Theorem \ref{thm:ub_correctness}; thus we are done.

\section{Proof of Proposition \ref{thm:lb-correctness}}
\thmLB*
\begin{proof}
Because any intervention $I_{P}[\Theta](\mathcal{N})$ gives a lower bound on the maximum probability of the evidence $e$, we know that the algorithm will provide a lower bound on the true value. Assuming the algorithm terminates, it will output a locally maximal intervention; this follows immediately from the termination criterion. It remains to show that the algorithm does indeed terminate. We show this by demonstrating that the algorithm corresponds to best-response dynamics on a potential game, for which it is known that best response dynamics converge to a Nash Equilibrium in potential games \citep{roughgarden2010algorithmic}. 

We construct the mapping as follows: for each node $V \in \intvVars$ and each instantiation $\parentsvalue$ of its parents $\parents$, we construct players $p_{V|\parentsvalue}$, yielding a finite set $\mathcal{P}$. The action set of each player $p_{V|\parentsvalue}$ is the finite set of deterministic distributions over the support of $V$. 
Recall that it suffices to consider only \textit{deterministic} interventions to solve the interventional robustness problem, and so this restriction to a finite action set does not prevent the algorithm from obtaining a globally optimal value.

The payoff to all players under strategy set $\allParam'_{\intvVars} = \{\cpt'_{V|\parentsvalue}| p_{V|\parentsvalue} \in \mathcal{P}\}$ is identical and is equal to $P_{\mathcal{N}[\allParam'_{\intvVars}]}(\evidence)$.
Because the players have identical payoffs, we can define a potential function for this game as this joint payoff function $P_{\mathcal{N}[\allParam'_{\intvVars}]}(\evidence)$, and observe that if a player follows best-response dynamics, it will necessarily increase the value of $p(e)$. We therefore obtain that the procedure will eventually converge to a local optimum. 

The problem of finding a Nash equilibrium lies in the complexity class PPAD \citep{roughgarden2010algorithmic}. In the worst-case, best-response dynamics can take time proportional to the size of the joint strategy space, which in this case will be exponential in the size of the intervention variable set $\intvVars$. While anecdotally we do not observe running times anywhere near the worst case, we further note that because the $\introb$ problem concerns itself with finding a maximum over a set, any element of that set, i.e. any intervention, serves as a witness to give a lower bound on the maximum intervention probability. Thus, the value $v$ of the current intervention is always a lower bound on the $\introb$ value, and so the algorithm can be prematurely halted after any period of time and still yield a valid lower bound.
\end{proof}

\section{Worked example} \label{apx:worked_example}

We will use the simplified car insurance network from Figure~\ref{fig:insurance-repeated} to illustrate the interventional robustness problem, and highlight how existing methods cannot provide accurate guarantees for this problem.

\begin{figure}
    \centering
    \includegraphics[width=\linewidth]{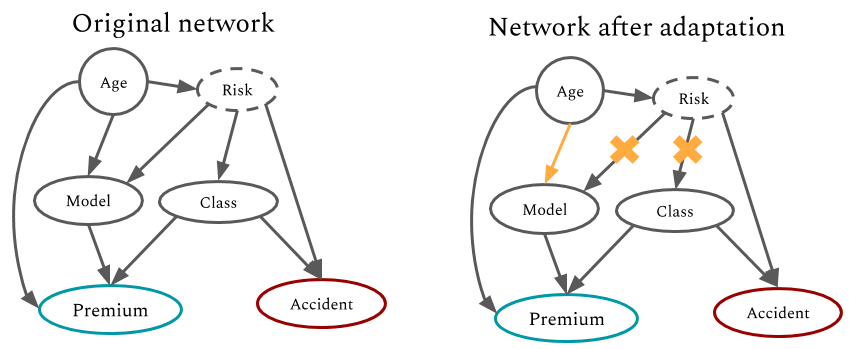}
    \caption{Causal model used by a fictional car insurance company in the example in Section~\ref{apx:worked_example}.}
    \label{fig:insurance-repeated}
\end{figure}
\textbf{Model:} we assume the probabilistic model of Figure \ref{fig:insurance-repeated} with CPTs described in Table~\ref{tab:insurance}. In particular, we assume that the insurance company decides on premiums based on some combination of the driver's age, car model, and history of taking driving courses. We will focus on the following decision rule.
\begin{align*}
    \texttt{low if (class}=1) \land &\texttt{[(age=<25} \land \texttt{model=budget)} \\
    &\lor \texttt{(age=>25 $\land$ model=lux)]}.
\end{align*}

\begin{table*}
\centering
\begin{tabular}{|c|c|c|}
    \hline
    P(\texttt{accident}=1) & P(\texttt{accident}=0) & $\parentsvalue$ \\
    \hline
      0.4 &0.6 &  Model=luxury  Class=1  risky=1\\ 
      	   0.01 & 0.99 &  Model=luxury  Class=1  risky=0\\ 
      	  0.6 & 0.4 &  Model=luxury  Class=0  risky=1\\ 
      	   0.1 & 0.9 &  Model=luxury  Class=0  risky=0\\ 
      	 0.3 & 0.7 &  Model=budget  Class=1  risky=1\\ 
      	   0.05 & 0.95 &  Model=budget  Class=1  risky=0\\ 
      	  0.5& 0.5 &  Model=budget  Class=0  risky=1\\ 
      	   0.3 & 0.7 &   Model=budget  Class=0  risky=0 \\
      	   \hline
\end{tabular}
\begin{tabular}{|c|c|c|}
\hline
$P(C|\parentsvalue)$ & $P(\bar{C}|\parentsvalue)$ & $\parentsvalue$ \\
\hline
 0.2 & 0.8 &risky=1\\
 0.8 & 0.2 & risky=0 \\
 \hline
\end{tabular}
\\
\begin{tabular}{|c|c|c|}
\hline
$P(\texttt{model}=\text{budget}|\parentsvalue)$ & $P(\texttt{model}=\text{luxury}|\parentsvalue)$ & $\parentsvalue$ \\
\hline
0.3& 0.7 & age$=>25$  risky=1\\ 
0.7& 0.3 & age=$>25$  risky=0\\ 
0.8& 0.2 & age=$<25$  risky=1\\ 
0.2& 0.8  & age=$<25$  risky=0\\
\hline
\end{tabular}
\begin{tabular}{|c|c|c|}
\hline
$P($risky=0) & $P($risky=1) & $\parentsvalue$ \\
\hline
 0.3 & 0.7  &    age=$>25$\\
 0.3 & 0.7 & age=$<25$ \\
 \hline
\end{tabular}
\caption{CPTs for the insurance network from Figure~\ref{fig:insurance-intervention}.}
\label{tab:insurance}
\end{table*}
The target probability in this setting will be $P(\evidence)$ with $\evidence = (\texttt{Accident}=1, \texttt{C1}=\texttt{low})$, and the intervention set will be $\intvVars = \{\texttt{model}, \texttt{class}\}$.
The intervention which maximizes $P(\evidence)$ is that in which $\cpt_{\text{model}|\text{age}=>25}' = \mathbbm{1}[\texttt{model}=\texttt{luxury}]$ and $\cpt_{\text{model}|\text{age}=<25}' = \mathbbm{1}[ \texttt{model}=\texttt{budget}]$, and $\cpt_{C|\parentsvalue}' = \mathbbm{1}[C=1]$. We will write these CPTs by $\cpt^*_{w|\parentsvalue}$. Under this intervention, the probability $P_{\bn[\allParam ']}(\evidence) = 0.126$. We set our threshold $\epsilon=0.1$, and so the classifier $\texttt{C1}$ is \textit{not} robust to this intervention set.

There are two critical components to answering the intervention robustness problem: the first is accurate estimation of the probability $P_{\bn[\allParam']}(\evidence)$, and the second is effective search over the intervention space. In what follows, we will show how two approximation methods based on existing methodology fail at each of these components, and therefore fail to effectively answer the intervention robustness problem. This motivates our proposed algorithms.

\textbf{Sensitivity analysis}, when performed exactly, will provide the correct answer to parametric $\introb$. However, existing approaches are computationally intractable for large intervention sets. For example, the method of \citet{ChanDarwiche04} requires differentiating the network polynomial with respect to every combination of parameters from the intervenable CPTs, which scales very poorly with respect to the size of each CPT and number of intervenable variables $|\intvVars|$. As a workaround, we consider a first-order approximation which consists of summing the partial derivatives of each CPT belonging to an intervenable variable $W \in \intvVars$.  This yields the following estimate of $\max_{\allParam '} P_{\bn[\allParam ']}(\evidence)$.
\begin{align*}
   \max_{\allParam '} &P_{\bn[\allParam']}(\evidence)  \approx \sum_{W \in \intvVars} \sum_{\parentsvalue(W)} \max_{w} \partial_{\cpt_{w|\parentsvalue}}P(\evidence) 
   \intertext{I.e. the approximation estimates the effect of parameter changes by computing the effect of single-CPT changes to $P(\evidence)$, and summing these estimates over all CPTs of interest.}
   &= \max_{\texttt{model}} \partial_{\cpt_{\texttt{model}|\texttt{age}>25, \cdot}}P(\evidence) + \max_{\texttt{model}} \partial_{\cpt_{\texttt{model}|\texttt{age}=<25, \cdot}}P(\evidence) \\ &+  \max_{\texttt{class}} \partial_{\cpt_{\texttt{class}|\cdot}}P(\evidence) \\
   &= P_{\bn[\cpt_{\texttt{class}|\parentsvalue}=\mathbbm{1}_{\texttt{class}=1}]}(\evidence) + P_{\bn[\cpt_{\text{model}|\parentsvalue}= \cpt^*_{\text{model}|\parentsvalue} ]}(\evidence) \\
   &= 0.048 + 0.038\\
   &=0.086 < \epsilon = 0.1
\end{align*}

Thus, under the single-parameter SA approximation, we would incorrectly estimate the classifier to be robust to interventions on model and class. 

We could also consider restricting our search over parametric interventions to one over \textbf{do-interventions only}. A do-intervention fixes the value of a variable, in contrast to a parametric intervention which changes its conditional distribution. The appeal of do-interventions is that they induce a much smaller search space; whereas the set of all parametric interventions can be of size $2^{2^{|\parents|}}$ (in a binary BN), the set of possible do-interventions is only of size $2^{|\intvVars|}$. The differential semantics of do-interventions are discussed by \citet{qin2015differential}, who show that it is possible to efficiently compute the effect of a do-intervention on a single variable in time linear in the size of the AC. 

In the CBN in Figure~\ref{fig:insurance-intervention}, $P_{\bn}(\evidence)$ is maximized when $\cpt_{\text{model}=\text{lux}|\text{age}=>25, \cdot }=1$ and $\cpt_{\text{model}=\text{budget}|\text{age}=<25, \cdot} = 1$ (i.e. when young people buy budget cars and old people buy luxury cars, independent of their risk appetite), and when all drivers take driving classes. While the latter change in the distribution can be expressed as a do-intervention, the former cannot. The value of $\texttt{model}$ must be set to either 0 or 1 for \textit{all} drivers. As a result, the worst-case do-intervention sets \texttt{class} to True and \texttt{model} to luxury. Under this intervention $\allParam'$, we see
\begin{equation*}
    P_{\bn[\allParam ']}(\evidence) = 0.0635 < 0.1 = \epsilon
\end{equation*}
and so considering only do-interventions is again not sufficient to answer the intervention robustness problem.

It should be noted that while both of these naive approaches are inspired by recent work which performs efficient inference using arithmetic circuits, it is not fair to call either a competing inference approach to our method, as the works cited developed their tools for different objectives. Rather, the purpose of this example is to highlight that many plausible approximation methods to verify the robustness of a classifier can fail on simple examples.

\section{Additional Circuit Evaluations}
\label{apx:additional-evals}
\subsection{Size of compiled circuits} \label{apx:additional-size}
\begin{table}[htb] 
\begin{tabular}{lllrrr}
\toprule
\textbf{Net} & \textbf{CSize} & \textbf{Ord} & \textbf{TW} & \textbf{AC size} & \makecell{\textbf{Time} \\ \textbf{(s)} }  \\ \midrule
insurance         & 0 (0)                            & N       &     29     & 362983          & 1.2     \\
         & 3 (41)                            & N       &    24      & 167121          & 0.5                         \\ 
         & 3 (41)                            & T       &    31     & 794267         & 4  \\
         & 3 (41)                            & S       &    33     & 1270075         & 8  \\ \midrule
child         & 0 (0)                            & N           &  15 & 4935          & 0.007   \\    
         & 8 (326)                            & N           &    38  & 234914          & 0.7                        \\ 
         & 8 (326)                            & T           &    38 & 1004786         & 1                        \\ \midrule
win95pts         & 0 (0)                            & N        &    18    & 17682          & 0.04  \\
         & 16 (799)                            & N        &    51    & 1210072          & 3                         \\ 
         & 16 (799)                            & T        &    58     & 52266950          & 77                         \\ \midrule
hepar2         & 12 (946)                            & N        &    53    & 8096874          & 49                         \\ 
        & 12 (946)                            & T        &    51     & 123108407          & 73                         \\
        & 12 (946)                            & S        &    51     & 123164181          & 75                         \\ \midrule
andes        & 12 (95)                            & N        &    41     & 24787127          & 272                         \\
        & 12 (95)                            & PT        &    43     & 60865146          & 778                         \\ 
\bottomrule
\end{tabular}
\caption{AC sizes and times (s) for the joint compilations used in the \ubn and \lbn algorithms. Shown are the number of input features $d$ and the sizes of the Boolean circuits representing the classifier (0 indicates no classifier), ordering constraints (none, partial topological, topological, or \struc topological), treewidth of the combined CNF encoding, and size and compilation time. We note that the large increase in network treewidth when adding a classifier is due to the treewidth of the classifier.}
\label{tab:ac_size}
\end{table}
In Table \ref{tab:ac_size}, we show details from joint compilations on five Bayesian networks and decision rules on those networks (further details in Appendix \ref{apx:exp_details}). First, we notice that the actual size of the compiled ACs is much smaller than that given by the worst case bound (which is exponential in treewidth), due to optimizations in the C2D compiler. Second, notice also that the size of the AC does not increase by $2^d$ when adding a decision rule to the classification: for instance, in win95pts, the size increases by a factor of $\sim 400$, while the decision rule has $2^{16} = 65536$. Interestingly, for the insurance network, the AC size actually \textit{decreases} when compiling with a small decision rule; this is likely due to good fortune with the min-fill heuristic. Finally, when we enforce a topological ordering, the size of the compilation increases, but not by more than $\sim 100$. Remarkably, this allows us to upper bound a marginal probability against parametric intervention sets involving any number of intervenable nodes. Our results provide evidence that our methods can scale to fairly large networks and classifiers, including with topological and \struc topological orderings.

The ordering types are defined as follows. Recall that, for the correctness of the upper bounding algorithm, we require $\var_j < \var_i$ for all $\var_j \in \Pa(\var_i)$, and for all $\var_i \in \intvVars$. Partial topological orderings impose these constraints $\var_j < \var_i$. Topological orderings impose the constraints for all $\var_i \in \vars$, rather than just $\var_i \in \intvVars$. This is generally preferred if computationally feasible, as it allows us to compute upper bounds for any parameter intervention set, and further also tends to produce better bounds. However, for the \texttt{andes} network, we found that this was too computationally demanding. \Struc ordering constraints contain topological constraints, and further for every $\var_i \in \intvVars$, we add $V_j < V_i$ for $V_j \in \Pa_{\graph'}(V_i)$, so that we can compute upper bounds on \struc intervention sets. 

\subsection{Tightness of Lower and Upper Bounds}
\label{apx:tightness}
We provide additional results on the tightness of our upper and lower bounds on an expanded set of evaluations, including both false positives and false negatives as evidence, and considering a broader range of networks and intervention sets. Results are detailed in Table~\ref{tab:tightness-expanded}.

\begin{table*}[h]
\centering
\begin{tabular}{llllllllll}
\toprule
\multirow{2}{*}{Network} & \multirow{2}{*}{IntSet} & \multicolumn{4}{c}{False Negatives} & \multicolumn{4}{c}{False Positives} \\ \cmidrule(lr){3-6} \cmidrule(lr){7-10}
&  & \multicolumn{1}{c}{BeforeIntv} & \multicolumn{1}{c}{LBound} & \multicolumn{1}{c}{RBound} & $\Delta$ & \multicolumn{1}{c}{BeforeIntv} & \multicolumn{1}{c}{LBound} & \multicolumn{1}{c}{UBound} & $\Delta $\\
 \midrule
child            & P1                &0.06922   &0.07098 & 0.07098 & \textbf{0} &0.1629              &       0.1947            & 0.1947  &  \textbf{0}      \\ 
            & P2                &0.06922  &0.07325    &0.07329 & \textbf{0.00004} & 0.1629              &       0.2762            & 0.3069    & \textbf{0.0307}         \\ 
           & P3                 &0.06922  &0.06978   &0.07127 & \textbf{0.00149} & 0.1629              &       0.1717            & 0.2009     & \textbf{0.0292}    \\ \midrule
insurance       & P1                & 0.02453             &       0.1181            & 0.1276 & \textbf{0.0095}  & 0.1981  &0.4157 &0.4161  & \textbf{0.0004}     \\ 
        & P2                & 0.02453             &       0.3275            & 0.3433  & \textbf{0.0158} & 0.1981     & 0.9123 & 0.9130 & \textbf{0.0007}   \\
       & P3                & 0.02453             &       0.02453            & 0.02453 & \textbf{0}   & 0.1981     & 0.1981   & 0.1981  &\textbf{0}  \\ 
        & S1                & 0.02453             &       0.1181            & 0.1297 & \textbf{0.0116}  &   0.1981 &  0.4157 &  0.4168 & \textbf{0.0011}  \\ \midrule
win95pts         & P1                & 0.2106              &       0.2111          & 0.2111  & \textbf{0}    & 0.005170    & 0.005416  &  0.005445  & \textbf{0.000029}  \\ 
         & P2                & 0.2106          &       0.2163            & 0.2191  & \textbf{0.0028}       & 0.005170 & 0.007200 &  0.008665 & \textbf{0.001465}   \\

        & P3                & 0.2106              &       0.2972           & 0.2985 & \textbf{0.0013} & 0.005170  &0.01430     &0.01445   & \textbf{0.00015}      \\

       & P4                & 0.2106              &       0.2109            & 0.2117 & \textbf{0.0008}  & 0.005170  & 0.05494     & 0.05674    & \textbf{0.00180}        \\
\midrule
hepar2         & P1                & 0.03673   & 0.09445     & 0.09445 & \textbf{0} & 0.2360              &       0.2408           & 0.2408   & \textbf{0}        \\

        & P2                & 0.03673 & 0.09585     & 0.09585 & \textbf{0} & 0.2360              &       0.9041           & 0.9041   & \textbf{0}         \\

        & P3               & 0.03673   & 0.1029    & 0.1029 & \textbf{0} & 0.2360              &       0.43758          & 0.43773    & \textbf{0.00015}        \\

        &  S1             & 0.03673  &  0.1029    & 0.1029 & \textbf{0} & 0.2360              &       0.43758          & 0.43793 & \textbf{0.00035}         \\
\midrule
andes         &  P1                & 0.001400             &       0.001400       & 0.002540  & \textbf{0.001140}   & 0&  0     & 0    & \textbf{0}   \\

         &  P2                & 0.001400             &       0.001400          & 0.002656  &\textbf{0.001256}   &0 &  0     & 0   & \textbf{0}    \\
\bottomrule

\end{tabular}
\caption{Analysis of the tightness of bounds (on probability of false negatives/false positives) produced by the \ubn and \lbn algorithms. For each network, we have different intervention sets. We show the probability of false negatives/false positives in the original Bayesian network (BeforeIntv), along with lower and upper bounds under each intervention set.
}
\label{tab:tightness-expanded}
\end{table*}

\section{Experimental Details} \label{apx:exp_details}

In our experiments, we use two types of classifiers $\df$. For the \texttt{child, win95pts, andes} networks, we use Bayesian network classifiers (BNC) trained on the respective networks, in order to predict a root node of the BN. For the \texttt{insurance, hepar2} network, since the chosen prediction targets are not root nodes of the networks, we used a Na\"ive Bayes classifier where the conditional probs have been extracted from BN. For \texttt{insurance, hepar2}, since the prediction targets $Y$ are not root nodes, we can test interventions which change the distribution on $Y$. Decision functions are obtained from the classifiers by applying a threshold.

To obtain Boolean circuits for the decision functions, we use a BNC-to-ODD compiler \citep{ShihEtAl19}, and then convert the ODD into a Boolean circuit (NNF). 

\subsubsection{Classifiers}

\begin{enumerate}
    \item \texttt{child}
    \begin{itemize}
        \item Target: \texttt{BirthAsphyxia}
        \item Features: \texttt{LVHreport, GruntingReport, XrayReport, LowerBodyO2, CardiacMixing, Age, RUQO2, CO2Report}
    \end{itemize}
    \item \texttt{insurance}
    \begin{itemize}
        \item Target: \texttt{MedCost}
        \item Features: \texttt{Age, MakeModel, DrivHist}
    \end{itemize}
    \item \texttt{win95pts}
    \begin{itemize}
        \item Target: \texttt{PTROFFLINE}
        \item Features: \texttt{Problem3, Problem2, PrtStatMem, PrtStatToner, Problem6, PrtFile, PrtStatOff, PrtIcon, Problem1, REPEAT, HrglssDrtnAftrPrnt, TstpsTxt, PSERRMEM, Problem5, Problem4, PrtStatPaper}
    \end{itemize}
    \item \texttt{hepar2}
    \begin{itemize}
        \item Target: \texttt{Steatosis}
        \item Features: \texttt{alt, triglycerides, ggtp, jaundice, alcohol, pain\_ruq, cholesterol, ESR, hepatalgia, ast, nausea, fat}
    \end{itemize}
    \item \texttt{Andes}
    \begin{itemize}
        \item Target: \texttt{TRY12}
        \item Features: \texttt{TRY15, SNode\_14, SNode\_19, TRY13, TRY14, GOAL\_99, SNode\_46, SNode\_31, SNode\_155, SNode\_123, SNode\_40, TRY26} 
    \end{itemize}
\end{enumerate}

\subsubsection{Intervention Sets}

\begin{enumerate}
    \item \texttt{child}
    \begin{itemize}
        \item P1: $\intvVars = \{\texttt{GruntingReport}\}$
        \item P2: $\intvVars = \{\texttt{ChestXray, Sick, Grunting}\}$
        \item P3: $\intvVars = \{\texttt{LowerBodyO2, RUQO2, CO2Report}\}$
    \end{itemize} 
    \item \texttt{insurance}
    \begin{itemize}
        \item P1: $\intvVars = \{\texttt{MakeModel, Cushioning}\}$
        \item P2: $\intvVars = \{$\texttt{SocioEcon, RiskAversion, Theft, Mileage, MakeModel, Cushioning}$\}$
        \item P3: $\intvVars = \{$\texttt{ThisCarDam, AntiTheft, OtherCarCost}$\}$
        \item C1: $\intvVars = \{\texttt{MakeModel, Cushioning}\}$, $C(\texttt{MakeModel}) = \{$\texttt{Age, AntiTheft, DrivHist, DrivingSkill, GoodStudent, HomeBase, Mileage, OtherCar, RiskAversion, SeniorTrain, SocioEcon, VehicleYear}, $\}$, 
        
        $C(\texttt{MakeModel}) = \{$\texttt{Age, Airbag, AntiTheft, Antilock, CarValue, DrivHist, DrivQuality, DrivingSkill, GoodStudent, HomeBase, MakeModel, Mileage, OtherCar, RiskAversion, RuggedAuto, SeniorTrain, SocioEcon, Theft, VehicleYear}, $\}$
    \end{itemize} 
    \item \texttt{win95pts}
    \begin{itemize}
        \item P1: $\intvVars = \{\texttt{NetOK, NetPrint}\}$
        \item P2: $\intvVars = \{$\texttt{AvlblVrtlMmry, DSApplctn, DskLocal, HrglssDrtnAftrPrnt, NtSpd, DeskPrntSpd, EPSGrphc, PSGRAPHIC, FllCrrptdBffr} $\}$
        \item P3: $\intvVars = \{\texttt{REPEAT}\}$
        \item P4: $\intvVars = \{$\texttt{GDIIN, PC2PRT, PSGRAPHIC, DS\_LCLOK, PSERRMEM, EMFOK, DS\_NTOK}$\}$
    \end{itemize} 
    \item \texttt{hepar2}
    \begin{itemize}
        \item P1: $\intvVars = \{\texttt{alcoholism}\}$
        \item P2: $\intvVars = \{\texttt{alcoholism}\}$
        \item P3: $\intvVars = \{$\texttt{alcoholism, hepatomegaly, alcohol, itching, fatigue, consciousness, hospital}$\}$
        \item C1: $\intvVars = \{$\texttt{alcoholism, hepatomegaly, alcohol, itching, fatigue, consciousness, hospital}$\}$
        
        The context function $C$ specifies additional parents on top of those present in the original BN. All $\intvVars$ have \texttt{age, sex, alcoholism} as additional parents. Further, \texttt{fatigue, consiciousness, hospital} have \texttt{anorexia} as additional parent. 
    \end{itemize}
    \item \texttt{andes}
    \begin{itemize}
        \item P1: $\intvVars = \{$\texttt{GOAL\_49, GOAL\_61, SNode\_26, SNode\_37}$\}$
        \item P2: $\intvVars = \{$\texttt{GOAL\_49, GOAL\_61, SNode\_26, SNode\_37, GOAL\_57, GOAL\_149, GOAL\_153, SNode\_74} $\}$
    \end{itemize} 
\end{enumerate}

\end{document}